\documentclass{article}

     \usepackage[nonatbib,final]{neurips_2020}
\usepackage[utf8]{inputenc} % allow utf-8 input
\usepackage[T1]{fontenc}    % use 8-bit T1 fonts
\usepackage{hyperref}       % hyperlinks
\usepackage{url}            % simple URL typesetting
\usepackage{booktabs}       % professional-quality tables
\usepackage{amsfonts}       % blackboard math symbols
\usepackage{nicefrac}       % compact symbols for 1/2, etc.
\usepackage{microtype}      % microtypography
\usepackage{color}
\usepackage{appendix}
\usepackage{enumitem} 
\usepackage{acronym}
\usepackage{psfrag}
\usepackage{amsthm}
\usepackage{mathtools}
\usepackage{algorithm,algpseudocode}
\usepackage{relsize}
\usepackage{setspace}
\usepackage{subfigure}
\usepackage{cite}

\theoremstyle{definition}

\newtheorem{theorem}{Theorem}

\newtheorem{lemma}{Lemma}

\newcommand\V[1]  { \mathbf{#1} }
\newcommand\B[1]  { \boldsymbol{#1} }
\newcommand\up[1] {\mathrm{#1}}
\newcommand\set[1] {\mathcal{#1}}
\acrodef{LPC}{linear probabilistic classifier}
\acrodef{RV}{random variable}
\acrodef{ERM}{empirical risk minimization}
\acrodef{RRM}{robust risk minimization}
\acrodef{SVM}{support vector machine}
\acrodef{ANN}{artificial neural network}
\acrodef{RKHS}{reproducing kernel Hilbert space}
\acrodef{MEM}{maximum entropy machine}
\acrodef{DT}{decision tree}
\acrodef{QDA}{quadratic discriminant analysis}
\acrodef{KNN}{k-nearest neighbor} 
\acrodef{RF}{random forest}
\acrodef{MLP}{multilayer perceptron}
\acrodef{MRC}{minimax risk classifier}
\acrodef{LR}{logistic regression}
\acrodef{LUSI}{learning using statistical invariants}
\acrodef{ACSC}{adversarial cost-sensitive classifier}
\acrodef{AMC}{adversarial multiclass classifier}

\title{Minimax Classification with 0-1 Loss \\and Performance Guarantees}

\author{%
Santiago Mazuelas \\
  BCAM-Basque Center of Applied Mathematics\\ and IKERBASQUE-Basque Foundation for Science\\
  Bilbao, Spain\\
  \texttt{smazuelas@bcamath.org} \\
  \And
  Andrea Zanoni\\
  \'{E}cole Polytechnique F\'{e}d\'{e}rale de Lausanne\\  
  Lausanne, Switzerland\\
  \texttt{andrea.zanoni@epfl.ch}
   \And
 Aritz P\'{e}rez\\
  BCAM-Basque Center of Applied Mathematics\\
  Bilbao, Spain\\
  \texttt{aperez@bcamath.org} \\
  }

\begin{document}
\maketitle
\begin{abstract}
%\vspace{-0.2cm}
Supervised classification techniques use training samples to find classification rules with small expected 0-1 loss. Conventional methods achieve efficient learning and out-of-sample generalization by minimizing surrogate losses over specific families of rules. This paper presents minimax risk classifiers (MRCs) that do not rely on a choice of surrogate loss and family of rules. MRCs achieve efficient learning and out-of-sample generalization by minimizing worst-case expected 0-1 loss w.r.t. uncertainty sets that are defined by linear constraints and include the true underlying distribution. In addition, MRCs' learning stage provides performance guarantees as lower and upper tight bounds for expected 0-1 loss. We also present MRCs' finite-sample generalization bounds in terms of training size and smallest minimax risk, and show their competitive classification performance w.r.t. state-of-the-art techniques using benchmark datasets.
\end{abstract}
%\vspace{-0.2cm}
\section{Introduction}
%\vspace{-0.2cm}
Supervised classification techniques use training samples to find classification rules that assign labels to instances with small expected 0-1 loss, also referred to as risk or probability of error. Most learning methods utilize \ac{ERM} approach that minimizes the expectation w.r.t. the empirical distribution of training samples, see e.g., \cite{Vap:98,EvgPonPog:00}. Other methods utilize \ac{RRM} approach that minimizes the worst-case expectation w.r.t. an uncertainty set of distributions obtained using metrics such as moments' fits, divergences, and Wasserstein distances, see e.g., \cite{LanGhaBhaJor:02,LeeRag:18}. Common uncertainty sets are formed by distributions with instances' marginal supported on the training samples  \cite{FarTse:16,AsiXinBeh:15,FatAnq:16,DucGlyNam:16,NamDuc:17}. However, more general uncertainty sets, such as those used in  \cite{LanGhaBhaJor:02,DelYe:10,AbaMohKuh:15,ShaKuhMoh:17,LeeRag:18}, can include the true underlying distribution with a tuneable confidence. Out-of-sample generalization is conventionally achieved by considering families of rules with favorable properties (reduced VC dimension or Rademacher complexity \cite{Vap:98,MehRos:18}). However, \ac{RRM} techniques can directly achieve out-of-sample generalization by using uncertainty sets that include the true underlying distribution. In addition, such uncertainty sets can enable to obtain tight performance bounds at learning.

Conventional methods achieve efficient learning and out-of-sample generalization by minimizing surrogate losses over families of rules with favorable properties. \ac{ERM}-based techniques such as \acp{SVM},  \acp{MLP}, and Adaboost classifiers consider loss functions such as hinge loss, cross-entropy loss, and exponential loss together with families of classification rules obtained from \acp{RKHS}, artificial neural networks, and combinations of weak rules. \ac{RRM}-based techniques that utilize Wasserstein distances consider surrogate log loss and linear functions or \acp{RKHS} \cite{AbaMohKuh:15,ShaKuhMoh:17}, while those that utilize f-divergences can use more general surrogate losses and parametric families of rules as long as they result in convex functions over parameters \cite{DucGlyNam:16,NamDuc:17}. Certain techniques based on \ac{RRM} do not rely on surrogate losses and minimize worst-case 0-1 expected loss \cite{AsiXinBeh:15,FarTse:16,FatAnq:16}. However, such works consider uncertainty sets that do not include the true underlying distribution. Hence, their generalization guarantees rely on the usage of specific families of rules, and they do not provide performance bounds at learning.

This paper presents \ac{RRM}-based classification techniques referred to as \acp{MRC} that minimize worst-case expected 0-1 loss over general classification rules, and provide tight performance bounds at learning.   
%This paper presents classification techniques that do not rely on a choice of surrogate loss and family of rules
%, and
%still achieve efficient learning and out-of-sample generalization. 
Specifically, the main results presented in the paper are as follows.\vspace{-0.2cm}
\begin{itemize}[leftmargin=0.7cm]\itemsep0em
\item Learning techniques that determine \acp{MRC} as the solution of a linear optimization problem (Theorem~\ref{th1} in Section~\ref{sec-2}, and Algorithm~\ref{codes1} in Section~\ref{sec-4}). %(Theorem~\ref{th1} in Section~\ref{sec-2}, and Algorithm~\ref{codes1} in Section~\ref{sec-4}). 
\item Techniques that provide performance guarantees at learning as lower and upper tight bounds for expected 0-1 loss (Theorem~\ref{th1} in Section~\ref{sec-2}, Theorem~\ref{prop} in Section~\ref{sec-3}, and Algorithm~\ref{codes1} in Section~\ref{sec-4}). %(Theorem~\ref{th1} in Section~\ref{sec-2}, Theorem~\ref{prop} in Section~\ref{sec-3}, and Algorithm~\ref{codes1} in Section~\ref{sec-4}).
\item Finite-sample generalization bounds for \acp{MRC} in terms of training size and smallest minimax risk
(Theorem~\ref{th-bounds} in Section~\ref{sec-3}).
\end{itemize}
Detailed comparisons with related techniques are provided in the remarks to the paper's main new results. In addition, Section~\ref{sec-4} provides a detailed description of  \acp{MRC}' implementation, and Section~\ref{sec-5} shows the suitability of the performance bounds and compares the classification error of \acp{MRC} w.r.t. state-of-the-art techniques.

\emph{Notation:} calligraphic upper case letters denote sets; 
%real-valued functions and vector-valued functions are denoted by lower and upper case letters, respectively; 
vectors and matrices are denoted by bold lower and upper case letters, respectively; for a vector $\V{v}$, $v^{(l)}$ denotes its 
$l$-th component, and $\V{v}^{\text{T}}$ and $\V{v}_+$ denote its transpose and positive part, respectively;  probability distributions and classification rules are denoted by upright fonts, e.g., $\up{p}$ and $\up{h}$; %and $\|\V{v}\|_{q,r}$and $(q,r)$-mixed norm of vector $\V{v}$,\footnote{The $(q,r)$-mixed norm of a vector $\V{v}\in\mathbb{R}^{I\cdot J}$ indexed by $\{1,2,\ldots,I\}\times\{1,2,\ldots,J\}$ is $\|\V{v}\|_{q,r}=\|\left[\|\V{v}_1\|_q,\|\V{v}_2\|_qx,\ldots,\|\V{v}_I\|_q\right]^{\text{T}}\|_r$ where $\V{v}_i=[v_{(i,1)},v_{(i,2)},\ldots,v_{(i,J)}]^{\text{T}}\in\mathbb{R}^J$ for $i=1,2,\ldots,I$. For instance, $\|\V{v}\|_{1,\infty}=\max_{i\in\set{I}}\sum_{j\in\set{J}}|\V{v}_{(i,j)}|$.} 
$\mathbb{E}_{\up{p}}\{\cdot\}$ denotes expectation w.r.t. probability distribution $\up{p}$; $\mathbb{I}\{\cdot\}$ denotes the indicator function; $\preceq$ and $\succeq$ denote vector (component-wise) inequalities; $\V{1}$ denotes a vector with all components equal to $1$;  $\V{1}_{\set{C}}$ an indicator vector with $j$-th component equal to $1$ (resp. $0$) if $j\in\set{C}$ (resp. $j\notin\set{C}$); $|\set{Z}|$ denotes de cardinality of set $\set{Z}$;  %We represent vector-valued functions with finite domains by matrices; specifically, we represent a vector function $F:\set{Z}\to\mathbb{R}^m$ by matrix $\V{F}\in\mathbb{R}^{|\set{Z}|\times m}$ with row $i$ given by $F(z_i)^\text{T}$ for $i=1,2,\ldots,|\set{Z}|$. In addition, if $F$ is a function with domain $\set{X}\times\set{Y}$, for each $x\in\set{X}$ we represent by $F_x$ the function $F_x(y)=F(x,y)$ with domain $\set{Y}$. 
%We represent real-valued and vector-valued functions with finite domains by vectors and matrices, respectively; specifically, we represent a function $f:\set{Z}\to\mathbb{R}$ for finite set $\set{Z}=\{z_1,z_2,\ldots,z_k\}$ with vector $\V{f}=[f(z_1),f(z_2),\ldots,f(z_k)]^\text{T}\in\mathbb{R}^k$, and a vector function $F:\set{Z}\to\mathbb{R}^m$ by matrix $\V{F}\in\mathbb{R}^{m\times k}$ with column $i$ given by $F(z_i)$ for $i=1,2,\ldots,k$. In addition, if $F$ is a function with domain $\set{X}\times\set{Y}$, for each $x\in\set{X}$ we represent by $F_x$ the function $F_x(y)=F(x,y)$ with domain $\set{Y}$. 
and, for a finite set $\set{Z}$, we denote by $\Delta(\set{Z})$ the set of probability distributions with support $\set{Z}$.
%\vspace{-0.2cm}
\section{Minimax-risk classification}\label{sec-2}
%\vspace{-0.2cm}
This section first briefly recalls the problem statement and learning approaches for supervised classification, and then presents learning techniques for \acp{MRC}.% for uncertainty sets defined by linear constraints. 
%\vspace{-0.1cm}
\subsection{Problem formulation and learning approaches}
%\vspace{-0.1cm}
Supervised classification uses training samples formed by instance-label pairs to determine classification rules that assign labels to instances.  In what follows, we denote by $\set{X}$ and $\set{Y}$ the sets of possible instances and labels, respectively; both sets are taken to be finite and we represent $\set{Y}$ by $\{1,2,\ldots,|\set{Y}|\}$. Commonly, the cardinality of $\set{X}$ is very large compared 
with that of $\set{Y}$; for instance, in hand-written digit classification with 28x28 pixels grayscale images, $|\set{X}|=256^{784}$ and $|\set{Y}|=10$. 

Classification rules can be deterministic or non-deterministic. For a specific instance, a deterministic classification rule assigns always the same label, while a non-deterministic classification rule is allowed to randomly assign a label with certain probability. Both types of rules can be represented by the probabilities with which labels are assigned to instances ($0$ or $1$ probabilities for the deterministic case). 
%Let $\set{X}$ be the set of possible instances and $\set{Y}$ be the set of possible labels so that pairs instance-label belong , . %The former type of rules can be seen as a particular case of the later in which, for each instance, labels are assigned with $1$ or $0$ probabilities.
%Therefore, general classification rules can be seen as probabilistic transformations also known as Markov transitions \cite{RooWill:18}. 
We denote by $ T(\set{X},\set{Y})$ the set of general classification rules; if $\up{h}\in T(\set{X},\set{Y})$ we denote by $\up{h}(y|x)$ the probability with which $\up{h}$ assigns label $y\in\set{Y}$ to instance $x\in\set{X}$. In addition, we denote by $\Delta(\set{X}\times\set{Y})$ the set of probability distributions on $\set{X}\times\set{Y}$; if $\up{p}\in\Delta(\set{X}\times\set{Y})$ we denote by $\up{p}(x,y)$ the probability assigned by $\up{p}$ to the instance-label pair $(x,y)$, and by $\up{p}(x)$ the marginal probability assigned by $\up{p}$ to the instance $x$, i.e., $\up{p}(x)=\underset{y\in\set{Y}}{\sum}\up{p}(x,y)$.

The $0$-$1$ loss (also called just loss in the following) of a classification rule at the instance-label pair $(x,y)\in\set{X}\times\set{Y}$ quantifies classification error, that is, the loss is $0$ if the classification rule assigns label $y$ to instance $x$, and is  $1$ otherwise. Hence, the expected loss of a classification rule $\up{h}\in T(\set{X},\set{Y})$ at $(x,y)$ is $1-\up{h}(y|x)$, and its expected loss w.r.t. a probability distribution $\up{p}\in\Delta(\set{X}\times\set{Y})$ is $$\ell(\up{h},\up{p})=\sum_{x\in\set{X},y\in\set{Y}}\up{p}(x,y)(1-\up{h}(y|x)).$$ 
Let $\up{p}^*$ be the unknown true underlying distribution of instance-label pairs, the risk of a classification rule $\up{h}$ (denoted $R(\up{h})$) is its expected loss w.r.t. $\up{p}^*$, that is
%The risk of a classification rule $h$ (denoted $R(h)$) is its expected loss w.r.t. the unknown true data-generating distribution of instance-label pairs denoted by $p^*$, that is
 $R(\up{h})=\ell(\up{h},\up{p}^*)$. The minimum risk is known as Bayes risk and becomes $$R_{\text{Bayes}}= 1-\sum_{x\in\set{X}}\max_{y\in\set{Y}}\up{p}^*(x,y)$$ since it is achieved by Bayes' rule $\up{h}_{\text{Bayes}}$ that assigns the most probable label to each instance. %with highest probability classifies each $x\in\set{X}$ with a label attaining the maximum of $p_x^*$
 
\ac{ERM} approach for supervised classification aims to minimize the empirical expected loss $\ell(\up{h},\up{p}_n)$, where $\up{p}_n$ is the empirical distribution of training samples. \ac{RRM} approach aims to minimize the worst-case expected loss $\ell(\up{h},\up{p})$ for $\up{p}$ a probability distribution in an uncertainty set obtained from training samples. As described above, conventional techniques enable efficient \ac{ERM} and \ac{RRM} by using surrogate loss functions and considering specific families of classification rules. 

Supervised classification techniques can be seen as methods that perform the approximation 
\begin{align*}
\underset{\up{h}\in T(\set{X},\set{Y})}{\min}\ell(\up{h},\up{p}^*)\longrightarrow\,\underset{\up{h}\in\set{F}}{\min}\max_{\up{p}\in\set{U}}\widetilde{\ell}(\up{h},\up{p})
\end{align*}
where the original 0-1 loss $\ell$ is substituted by a surrogate loss $\widetilde{\ell}$; classification rules are restricted to a specific family $\set{F}\subseteq T(\set{X},\set{Y})$; and expectation w.r.t. the true underlying distribution $\up{p}^*$ is approximated by the worst-case expectation w.r.t. distributions in an uncertainty set $\set{U}$. \ac{ERM}-based techniques correspond to the case where the uncertainty set contains only the empirical distribution, while \ac{RRM}-based techniques use uncertainty sets that contain multiple distributions. Using 0-1 loss and uncertainty sets that include the true underlying distribution, the objective minimized at learning $\max_{\up{p}\in\set{U}}\ell(\up{h},\up{p})$ becomes an upper bound of the original objective $\ell(\up{h},\up{p}^*)$ for any classification rule $\up{h}\in T(\set{X},\set{Y})$. This key property can enable to ensure out-of-sample generalization and to obtain tight performance bounds at learning.

% \ac{RRM} techniques considering uncertainty sets that include the true data-generating distribution have the key property of upper-bounding expected losses. Specifically, in such cases, the objective minimized at training $\max_{p\in\set{U}}\ell(h,p)$ becomes an upper bound of the original objective $\ell(h,p^*)$ for any classification rule $h$. This fact enables to achieve out-of-sample generalization considering general classification rules $\set{F}= T(\set{X},\set{Y})$ as well as to compute tight performance bounds at training.

%The key advantage of \ac{RRM} techniques considering uncertainty sets that include the true data-generating distribution is that  the objective minimized in training $\max_{p\in\set{U}}\ell(h,p)$ becomes an upper bound of the original objective $\ell(h,p^*)$ for any classification rule $h$. 
%\vspace{-0.1cm}
\subsection{Learning \acp{MRC}}
%\vspace{-0.1cm}
The following shows how \ac{RRM} can be used with original 0-1 loss $\ell$, considering general classification rules $ T(\set{X},\set{Y})$, and using uncertainty sets that include the true underlying distribution $\up{p}^*$ with a tuneable confidence. 
%As mentioned above, Theorem~\ref{th1} enables supervised classification techniques that minimize worst-case expectated 0-1 loss over general classification rules while conventional techniques utilizes surrogate losses and consider specific families of classification rules. For instance, \ac{RRM} methods in ... determine
%$$h^*\in\min_{h\in\set{F}}\max_{p\in\set{U}}\widetilde{\ell}(h,p)$$
%for different choices of surrogate losses $\widetilde{\ell}$ such as log losses and families of classification rules $\set{F}$ such as those given by a \ac{RKHS}. The final performance of such classification rules depend on the adequacy of choices for surrogate loss $\widetilde{\ell}$ family of rules $\set{F}$, and uncertainty set $\set{U}$.
%Minimization of the 0-1 loss $\ell$ instead of a surrogate $\widetilde{\ell}$ and considering general classification rules $ T(\set{X},\set{Y})$ instead of a specific family $\set{F}\subset T(\set{X},\set{Y})$ can achieve better final performances that are quantified in terms of 0-1 loss. Note that, while Theorem~\ref{th1} obtain minimax classification rules over general classification rules, such minimax rules have a specific form given by \eqref{robust-act}

\acp{MRC} consider uncertainty sets of distributions defined by linear constraints obtained from expectation estimates of a feature mapping. Specifically, let $\Phi:\set{X}\times\set{Y}\to\mathbb{R}^m$ be a feature mapping, and $\V{a},\V{b}\in\mathbb{R}^m$ with $\V{a}\preceq\V{b}$ be lower and upper endpoints of interval estimates for the expectation of $\Phi$. We consider uncertainty sets of distributions 
\begin{align}\label{uncertainty}\set{U}^{\V{a},\V{b}}=\big\{\up{p}\in\Delta(\set{X}\times\set{Y}):\  \V{a}\preceq\mathbb{E}_{\up{p}}\{\Phi(x,y)\}\preceq \V{b}\big\}\end{align}
and we denote the minimax expected loss against uncertainty set $\set{U}^{\V{a},\V{b}}$ by $R^{\V{a},\V{b}}$, i.e.,\begin{align*}R^{\V{a},\V{b}}=\min_{\up{h}\in T(\set{X},\set{Y})}\,\max_{\up{p}\in\set{U}^{\V{a},\V{b}}}\ell(\up{h},\up{p}).\end{align*}

Such uncertainty sets include the true underlying distribution $\up{p}^*$ with probability at least $1-\delta$ as long as $\V{a}$ and $\V{b}$ define expectations' confidence intervals at level $1-\delta$, that is $$\mathbb{P}\{\V{a}\preceq\mathbb{E}_{\up{p}^*}\{\Phi(x,y)\}\preceq\V{b}\}\geq 1-\delta.$$ In this paper, we consider expectations' interval estimates obtained from empirical expectations of training samples $(x_1,y_1),(x_2,y_2),\ldots,(x_n,y_n)$ 
as 
\begin{align}\label{interval} &\V{a}_n=\B{\tau}_n-\frac{\B{\lambda}}{\sqrt{n}},\  \V{b}_n=\B{\tau}_n+\frac{\B{\lambda}}{\sqrt{n}},\  \mbox{for }\B{\tau}_n=\frac{1}{n}\sum_{i=1}^n\Phi(x_{i},y_{i})\end{align}
where $\B{\lambda}\succeq \V{0}$ determines the size of the interval estimates for different confidence levels.

%can be used to obtain point estimates $\B{\tau}_n$ and interval estimates $\V{a}_n$, $\V{b}_n$ for the features expectation as
% where $\B{\tau}_n$ is the sample average 
%$$$$

%vectors $\V{a},\V{b}\in\mathbb{R}^m$ with $\V{a}\preceq\V{b}$ and a function $\Phi:\set{X}\times\set{Y}\to\mathbb{R}^m$, we denote by $\set{U}_\Phi^{\V{a},\V{b}}$ the set of probability distributions
%\begin{align}\label{uncertainty}\set{U}_\Phi^{\V{a},\V{b}}=\{p\in\Delta(\set{X}\times\set{Y}):\  \V{a}\preceq\mathbb{E}_p\{\Phi(x,y)\}\preceq \V{b}\}.\end{align}
%We call function $\Phi$ the feature mapping, and vectors $\V{a}$ and $\V{b}$ the lower and upper endpoints of expectation interval estimates, e.g., confidence intervals from sample means. In addition, 

In the following, in order to get compact expressions we often denote functions with domain $\set{X}\times\set{Y}$ by vectors or matrices with $|\set{X}||\set{Y}|$ components or rows, respectively. We denote a probability distribution $\up{p}\in\Delta(\set{X}\times\set{Y})$ and a classification rule $\up{h}\in T(\set{X},\set{Y})$ by vectors $\V{p}$ and $\V{h}$ with components given by $\up{p}(x,y)$ and $\up{h}(y|x)$ for $(x,y)\in\set{X}\times\set{Y}$. In addition, we denote the feature mapping $\Phi:\set{X}\times\set{Y}\to\mathbb{R}^m$ by a matrix $\B{\Phi}$ with rows given by $\Phi(x,y)^{\text{T}}$ for $(x,y)\in\set{X}\times\set{Y}$. Also, we denote by $\V{p}_x$, $\V{h}_x$, and $\B{\Phi}_x$ the subvectors and submatrix of $\V{p}$, $\V{h}$, and $\B{\Phi}$ corresponding to a fixed $x\in\set{X}$, and if $\V{v}$ is a vector indexed by $\set{X}\times\set{Y}$ we denote by $\|\V{v}\|_{1,\infty}$ and $\|\V{v}\|_{\infty,1}$ the mixed norms $\|\V{v}\|_{1,\infty}=\max_{x\in\set{X}}\|\V{v}_{x}\|_1$ and $\|\V{v}\|_{\infty,1}=\sum_{x\in\set{X}}\|\V{v}_{x}\|_\infty$. With this vector notation we have that 
\begin{align*}\ell(\up{h},\up{p})=\V{p}^{\text{T}}(\V{1}-\V{h}),\  
\min_{\up{h}\in T(\set{X},\set{Y})}&\ell(\up{h},\up{p})=1-\|\V{p}\|_{\infty,1},\  
\mbox{ and }\mathbb{E}_{\up{p}}\{\Phi(x,y)\}=\B{\Phi}^{\text{T}}\V{p}.\end{align*}%&R_{\text{Bayes}}=1- \|\V{p}^*\|_{\infty,1},
%In particular, the uncertainty sets $\set{U}^{\V{a},\V{b}}$ are given by linear constraints since $p\in\set{U}^{\V{a},\V{b}}\Leftrightarrow \V{a}\preceq \B{\Phi}^\text{T}\V{p}\preceq\V{b}$.
Finally, whenever we use expectation point estimates, i.e., $\V{a}=\V{b}$, we drop $\V{b}$ from the superscripts, for instance we denote $\set{U}^{\V{a},\V{b}}$ for $\V{a}=\V{b}$ as $\set{U}^{\V{a}}$. %[notation for marginals? do we need the subscript here? define in appendix?]

The result below determines minimax classification rules with 0-1 loss against uncertainty sets given by \eqref{uncertainty}, which are referred to as \acp{MRC} in the following. 
%as well as the corresponding minimax expected loss.   
\begin{theorem}\label{th1}
Let $\Phi:\set{X}\times\set{Y}\to\mathbb{R}^m$, $\V{a}, \V{b}\in\mathbb{R}^{m}$ with $\set{U}^{\V{a},\V{b}}\neq\emptyset$, and $\B{\mu}_{a}^*,\B{\mu}_{b}^*$, $\nu^*$ be a solution of the convex optimization problem
%}\in\mathbb{R}^{m}$ , $\set{U}=\{P\in\Delta(\set{X}\times\set{Y}):\  \underline{\tau}_i\leq\mathbb{E}_P\{\Phi_i(x,y)\}\leq\overline{\tau}_i,\  i=1,2,\ldots,m\}$ and $\lambda_0^*\in\mathbb{R}$, $\overline{\B{\lambda}}^*,\underline{\B{\lambda}}^*\in\mathbb{R}^{m+1}$ a solution of optimization problem
\begin{align}\label{learning-ineq}\begin{array}{cl}\underset{\B{\mu}_{a},\B{\mu}_{b}\in\mathbb{R}^{m},\nu\in\mathbb{R}}{\min}&\V{b}^{\text{T}}\B{\mu}_{b}-\V{a}^{\text{T}}\B{\mu}_{a}-\nu\\
\mbox{s. t.}&\|(\B{\Phi}(\B{\mu}_{a}-\B{\mu}_{b})+(\nu+1)\V{1})_+\|_{1,\infty}\leq 1\\
&\B{\mu}_{a},\B{\mu}_{b}\succeq \V{0}. \end{array}\end{align}
If a classification rule
$\up{h}^{\V{a},\V{b}}\in\Delta(X,Y)$ satisfies, for each $x\in\set{X},y\in\set{Y}$,
\begin{align}\label{robust-act}\up{h}^{\V{a},\V{b}}(y|x)\geq \Phi(x,y)^{\text{T}}\B{\mu}^*+\nu^*+1\end{align}
with $\B{\mu}^*=\B{\mu}_{a}^*-\B{\mu}_{b}^*$,
then $$\up{h}^{\V{a},\V{b}}\in\arg\,\min_{\up{h}\in\Delta(X,Y)}\,\max_{\up{p}\in\set{U}^{\V{a},\V{b}}}\ell(\up{h},\up{p})$$
that is, $\up{h}^{\V{a},\V{b}}$ is a minimax classification rule for 0-1 loss against uncertainty set $\set{U}^{\V{a},\V{b}}$. In addition, the minimax expected loss against uncertainty set $\set{U}^{\V{a},\V{b}}$ is given by \begin{align}\label{upper}R^{\V{a},\V{b}}=\V{b}^{\text{T}}\B{\mu}_{b}^*-\V{a}^{\text{T}}\B{\mu}_{a}^*-\nu^*.\end{align}
\end{theorem}
%\vspace{-0.3cm}
\begin{proof}
See Appendix~\ref{proof-th1} in the supplementary material.
\end{proof}
%\vspace{-0.2cm}
The result above is obtained by using von Neumann's minimax theorem \cite{GruDaw:04} and Lagrange duality \cite{BoyVan:04}; in particular, parameters $\B{\mu}_{a}^*,\B{\mu}_{b}^*$, $\nu^*$ correspond to the Lagrange multipliers of constraints in \eqref{uncertainty}. 
As we describe in Section~\ref{sec-4}, Theorem~\ref{th1} enables \acp{MRC}' implementation in practice. Specifically, training samples serve to obtain expectation estimates $\V{a}$ and $\V{b}$ that are used to learn parameters $\B{\mu}^*,\nu^*$ by solving \eqref{learning-ineq}, which is equivalent to a linear optimization problem. Then,  those parameters are used in the prediction stage to assign label $y\in\set{Y}$ to instance $x\in\set{X}$ with probability $\up{h}^{\V{a},\V{b}} (y|x)$ satisfying \eqref{robust-act}. Even though \acp{MRC} minimize the worst-case risk over all possible rules; as shown in \eqref{robust-act}, they have a specific parametric form determined by a linear-affine combination of the feature mapping with coefficients obtained by solving \eqref{learning-ineq} at learning.  Therefore, the role of the feature mapping in the presented method is similar to that in conventional techniques such as \ac{SVM} and logistic regression.

Classification rules satisfying \eqref{robust-act} always exist since $\sum_{y\in\set{Y}}(\Phi(x,y)^{\text{T}}\B{\mu}^*+\nu^*+1)_+\leq 1$ for any $x\in\set{X}$ due to the constraints in \eqref{learning-ineq}.  In addition, in case of using expectation point estimates, i.e., $\V{a}=\V{b}$, the minimization solved at learning becomes \begin{align}\label{learning-eq}\begin{array}{cl}\underset{\B{\mu}\in\mathbb{R}^{m},\nu\in\mathbb{R}}{\min}&-\V{a}^{\text{T}}\B{\mu}-\nu\\
\mbox{s. t.}&\|(\B{\Phi}\B{\mu}+(\nu+1)\V{1})_+\|_{1,\infty}\leq 1\end{array}\end{align}
taking $\B{\mu}=\B{\mu}_a-\B{\mu}_b$.

%minimizing over $\B{\mu}_{\B{b}}-\B{\mu}_{\B{a}}$ the optimization problem solved in learning does not have positivity constraints and has $m$ less dimensions, that is \eqref{learning-ineq} becomes

The techniques proposed in  \cite{AsiXinBeh:15,FarTse:16,FatAnq:16} find minimax classification rules with 0-1 loss for uncertainty sets that are also defined in terms of expectations' fits. In particular, \cite{AsiXinBeh:15,FatAnq:16} utilize uncertainty sets of the form
\begin{align*}\set{U}=\big\{\up{p}\in T(\set{X},\set{Y}): \mathbb{E}_{\up{p}}\{\Phi(x,y)\}=\V{a}, \mbox{ and }\up{p}(x)=\up{p}_n(x),\  \forall x\in\set{X}\big\}\end{align*}
while \cite{FarTse:16} utilizes uncertainty sets of the form
\begin{align*}\set{U}=\big\{\up{p}\in T(\set{X},\set{Y}): \|\mathbb{E}_{\up{p}}\{\Phi(x,y)\}-\V{a}\|\leq\varepsilon, \mbox{ and }\up{p}(x)=\up{p}_n(x),\  \forall x\in\set{X}\big\}.\end{align*}
Such uncertainty sets only contain distributions with instances' marginal $\up{p}(x)$ that coincides with the empirical $\up{p}_n(x)$ so that they do not include the true underlying distribution for finite number of samples. Therefore, the techniques in \cite{AsiXinBeh:15,FarTse:16,FatAnq:16} cannot ensure out-of-sample generalization with general classification rules and do not provide performance bounds at learning such as those shown below in Theorem~\ref{prop} for \acp{MRC}.
%\vspace{-0.2cm}
\section{Performance guarantees}\label{sec-3}
%\vspace{-0.2cm}
This section characterizes the out-of-sample performance of \acp{MRC}. We first present techniques that provide tight performance bounds at learning, and then we show finite-sample generalization bounds for \acp{MRC}' risk in terms of training size and smallest minimax risk.
% \vspace{-0.1cm}
\subsection{Tight performance bounds}
%\vspace{-0.1cm}
The following result shows that the proposed approach also allows to obtain bounds for expected losses by solving linear optimization problems.
\begin{theorem}\label{prop}
Let $\Phi:\set{X}\times\set{Y}\to\mathbb{R}^m$, $\V{a}, \V{b}\in\mathbb{R}^{m}$ with $\set{U}^{\V{a},\V{b}}\neq\emptyset$ and $\kappa^{\V{a},\V{b}}(q)$ be given by
\begin{align}\label{lower}\begin{array}{ccll}\kappa^{\V{a},\V{b}}(q)&=\underset{\B{\mu}_{a},\B{\mu}_{b}\in\mathbb{R}^{m},\nu\in\mathbb{R}}{\min}&\V{b}^{\text{T}}\B{\mu}_{b}-\V{a}^{\text{T}}\B{\mu}_{a}-\nu\\&
\hspace{0.3cm}\mbox{ s. t.}&\B{\Phi}(\B{\mu}_{a}-\B{\mu}_{b})+\nu\V{1}\preceq \V{q}\\&
&\B{\mu}_{a},\B{\mu}_{b}\succeq \V{0} \end{array}\end{align}
for a function $q:\set{X}\times\set{Y}\to\mathbb{R}$.
Then, for any $\up{p}\in\set{U}^{\V{a},\V{b}}$ and $\up{h}\in T(\set{X},\set{Y})$ \begin{align}\label{bounds}0\leq-\kappa^{\V{a},\V{b}}(1-h)\leq\ell(\up{h},\up{p})\leq\kappa^{\V{a},\V{b}}(\up{h}-1)\leq 1.\end{align}
In addition, $\ell(\up{h},\up{p})=-\kappa^{\V{a},\V{b}}(1-\up{h})$ (resp. $\ell(\up{h},\up{p})=\kappa^{\V{a},\V{b}}(\up{h}-1)$) if $\up{p}$ minimizes (resp. maximizes) the expected loss of $\up{h}$ over distributions in $\set{U}^{\V{a},\V{b}}$. 
%\begin{align*}
%\mbox{ if }p^*\in\arg\min_{p\in\set{U}_\Phi^{\V{a},\V{b}}}\ell(h,p)\mbox{ and }
%\ell(h,p^*)=1-\kappa_\Phi^{\V{a},\V{b}}(h)\mbox{ if }p^*\in\arg\max_{p\in\set{U}_\Phi^{\V{a},\V{b}}}\ell(h,p)\end{align*}
%\item $1-\kappa_\Phi^{\V{a},\V{b}}(T)= R_\Phi^{\V{a},\V{b}}$ if $T$ is given by \eqref{robust-act} in Theorem\ref{th1} above.
\end{theorem}
%\vspace{-0.3cm}
\begin{proof}
See Appendix~\ref{proof-prop} in the supplementary material.
\end{proof}
%\vspace{-0.2cm}
%If $h$ is a minimax classification rule given by \eqref{robust-act}, the upper bound above is directly given by \eqref{upper}, that is, 
For an \ac{MRC} $\up{h}^{\V{a},\V{b}}$,  the upper bound above is directly given by \eqref{upper}, that is, 
$R^{\V{a},\V{b}}=\kappa^{\V{a},\V{b}}(\up{h}^{\V{a},\V{b}}-1)$. On the other hand, its lower bound, denoted by $L^{\V{a},\V{b}}$, requires to solve an additional linear optimization problem given by \eqref{lower} to obtain $L^{\V{a},\V{b}}=-\kappa^{\V{a},\V{b}}(1-\up{h}^{\V{a},\V{b}})$.

%In addition, we denote by $L_\Phi^{\V{a}_n,\V{b}_n}$ the lower bound for the expected error of $h^{\V{a}_n,\V{b}_n}$ against uncertainty set $\set{U}_\Phi^{\V{a}_n,\V{b}_n}$, that is $L_\Phi^{\V{a}_n,\V{b}_n}=-\kappa_\Phi^{\V{a}_n,\V{b}_n}(h^{\V{a}_n,\V{b}_n}-1)$.
%For an \ac{LPC} $h^{\V{a},\V{b}}$, the upper bound above is directly given by the learning phase, that is, $R_\Phi^{\V{a},\V{b}}$ given in \eqref{upper} equals $1-\kappa_\Phi^{\V{a},\V{b}}(h^{\V{a},\V{b}})$. On the other hand, the lower bound for $h^{\V{a},\V{b}}$ denoted as $L_\Phi^{\V{a},\V{b}}= 1+\kappa_\Phi^{\V{a},\V{b}}(-h^{\V{a},\V{b}})$ requires to solve an additional linear optimization problem.

The techniques proposed in \cite{DucGlyNam:16,AbaMohKuh:15,ShaKuhMoh:17}
obtain analogous upper and lower bounds corresponding with \ac{RRM} methods that use uncertainty sets defined in terms of f-divergences and Wasserstein distances. Such methods obtain classification rules by minimizing the upper bound of a surrogate expected loss while \acp{MRC} minimize the upper bound of the 0-1 expected loss (risk).
Note that the bounds for expected losses become risk's bounds if the uncertainty set includes the true underlying distribution. Such situation can be attained with a tuneable confidence using uncertainty sets defined by Wasserstein distances as in \cite{AbaMohKuh:15,ShaKuhMoh:17} or using the proposed uncertainty sets in \eqref{uncertainty} with expectation confidence intervals. However, the bounds are only  asymptotical risk's bounds using uncertainty sets defined by f-divergences as in \cite{DucGlyNam:16} or using the proposed uncertainty sets in \eqref{uncertainty} with expectation point estimates.
%\vspace{-0.1cm}
\subsection{Finite-sample generalization bounds}
%\vspace{-0.1cm}
The smallest minimax risk using uncertainty sets given by \eqref{uncertainty} with feature mapping $\Phi$ is the non-random constant $R^{\B{\tau}_\infty}$ with $\B{\tau}_\infty=\mathbb{E}
_{\up{p}^*}\{\Phi\}$ because 
$\up{p}^*\in\set{U}^{\V{a},\V{b}}\Rightarrow \set{U}^{\B{\tau}_\infty}\subseteq\set{U}^{\V{a},\V{b}}\Rightarrow R^{\B{\tau}_\infty}\leq R^{\V{a},\V{b}}.$ Such smallest minimax risk corresponds with \ac{MRC} $\up{h}^{\B{\tau}_\infty}$ that would require an infinite number of training samples to exactly determine the features' actual expectation $\B{\tau}_\infty$.%_\infty=\mathbb{E}_{p^*}\{\Phi\}$.

The following result bounds the risk of \acp{MRC} w.r.t. the smallest minimax risk, as well as the difference between the risk of \acp{MRC} and the corresponding minimax expected loss. 
\begin{theorem}\label{th-bounds}
Let %$(x_{1},y_{1}),(x_{2},y_{2}),\ldots,(x_{n},y_{n})$ be $n$ training samples following distribution $p^*$, 
$\Phi:\set{X}\times\set{Y}\to\mathbb{R}^m$ be a feature mapping, $\delta\in(0,1)$ , and $\B{\tau}_\infty=\mathbb{E}_{\up{p}^*}\{\Phi\}$. If $\B{\tau}_n$, $\V{a}_n$, and $\V{b}_n$ are point and interval estimates for $\B{\tau}_\infty$ obtained from training samples as given by \eqref{interval} with 
%\begin{align}\label{interval} \B{\tau}_n=\frac{1}{n}\sum_{i=1}^n\Phi(x_{i},y_{i}),\  \V{a}_n=\B{\tau}_n-\V{s}\sqrt{\frac{1}{n}},\  \V{b}_n=\B{\tau}_n+\V{s}\sqrt{\frac{1}{n}}\end{align}
%with 
\begin{align*}\B{\lambda}=\V{d}\sqrt{\frac{\log m+\log\frac{2}{\delta}}{2}},\  d^{(l)}=\max_{x\in\set{X},y\in\set{Y}} \Phi(x,y)^{(l)}-\min_{x\in\set{X},y\in\set{Y}}\Phi(x,y)^{(l)}\mbox{, for $l=1,2,\ldots,m$}.\end{align*}
%\gamma\in\mathbb{R}:\  \gamma,\B{\lambda}\mbox{ is solution of \eqref{learning-eq} for some }\V{a}\in\text{Conv}(\Phi(\set{X}\times\set{Y})\}$$
%Then
Then, with probability at least $1-\delta$% $T^{\V{a}_n,\V{b}_n}$ is calibrated and
%\vspace{0.2cm}
\begin{align}
%R(h^{\V{a}_n,\V{b}_n})\leq R^{\V{a}_n,\V{b}_n}\label{bound0}\\[0.2cm]
R(\up{h}^{\V{a}_n,\V{b}_n})&\leq R^{\V{a}_n,\V{b}_n}\leq R^{\B{\tau}_\infty}+2M_\Phi\|\V{d}\|_2\sqrt{\frac{\log m+\log\frac{2}{\delta}}{2}}\frac{1}{\sqrt{n}}\label{bound1}\\[0.2cm]
%&R(h^{\B{\tau}_n})\geq L_\Phi^{\B{\tau}_n}-M_\Phi\|\V{c}\|_2\sqrt{\frac{\log m+\log\frac{2}{\delta}}{2}}\frac{1}{\sqrt{n}}\label{bound3}\\
R(\up{h}^{\B{\tau}_n})&\leq R^{\B{\tau}_n}+M_\Phi\|\V{d}\|_2\sqrt{\frac{\log m+\log\frac{2}{\delta}}{2}}\frac{1}{\sqrt{n}}\label{bound2}\\[0.2cm]
R(\up{h}^{\B{\tau}_n})&\leq R^{\B{\tau}_\infty}+N_\Phi\|\V{d}\|_2\sqrt{\frac{\log m+\log\frac{2}{\delta}}{2}}\frac{1}{\sqrt{n}}\label{bound4}\end{align}%\vspace{0.2cm}
where
\begin{align*}&\,\,M_\Phi\,=\max_{\B{\mu}\in\Omega_\Phi}\|\B{\mu}\|_2,\  
N_\Phi=\max_{\B{\mu}_1,\B{\mu}_2\in\Omega_\Phi}\|\B{\mu}_1-\B{\mu}_2\|_2\\&
\begin{array}{ll}\Omega_\Phi\  =\big\{\hspace{-0.3cm}&\B{\mu}\in\mathbb{R}^m:\  \exists \V{a}\in\text{Conv}(\Phi(\set{X}\times\set{Y}))
  \mbox{   s.t. }  \B{\mu},\nu\mbox{ is the min. euclidean norm solution of \eqref{learning-eq}}\big\}.\end{array}\end{align*}
%$$|R_\Phi^{\B{\tau}_n}-R_\Phi^{\B{\tau}^*}|\leq M_\Phi\|\V{c}(\delta)\|_\infty\frac{1}{\sqrt{2n}}$$
%In addition, if \eqref{learning-eq} for $\V{a}=\B{\tau}_\infty$ has unique solution, then % then with probability $1$
%$$R(h^{\V{a}_n,\V{b}_n})\underset{n\to\infty}{\to}R(h^{\B{\tau}_\infty})$$
%$$R(h^{\B{\tau}_n})\underset{n\to\infty}{\to}R(h^{\B{\tau}_\infty})$$
%With probability $1$,x
%$$|R_\Phi^{\V{a}_n,\V{b}_n}- R_\Phi^{\B{\tau}^*}|\leq M_\Phi\|\B{\kappa}(\delta)\|_\infty\frac{1}{\sqrt{2n}}$$
%and, if solution of \eqref{learning-eq} for $\V{a}=\B{\tau}^*$ is unique, we have that 
%$$R(T^{\V{a}_n,\V{b}_n})\underset{n\to\infty}{\to}R(T^{\B{\tau}^*})$$
\end{theorem}
%\vspace{-0.3cm}
\begin{proof}
See Appendix~\ref{proof-bounds} in the supplementary material.
\end{proof}
%\vspace{-0.2cm}
Second inequality in \eqref{bound1} and inequality \eqref{bound4} bound the risk of \acp{MRC} w.r.t. the smallest minimax risk $R^{\B{\tau}_\infty}$; and first inequality in \eqref{bound1} and inequality \eqref{bound2} bound the difference between the risk of \acp{MRC} and the corresponding minimax expected loss. These bounds show differences that decrease with $n$ as $O(1/\sqrt{n})$ with proportionality constants that depend on the confidence $\delta$, and other constants describing the complexity of feature mapping $\Phi$ such as its dimensionality $m$, the difference between its maximum and minimum values $\V{d}$, and bounds for the solutions of \eqref{learning-eq} with vectors $\V{a}$ in the convex hull of $\Phi(\set{X}\times\set{Y})$. %\color{red} Such decrease in risk can also be described in terms of the size of the uncertainty set $\set{U}^{\V{a}_n,\V{b}_n}$ considered since
%$$\V{b}_n-\V{a}_n=2\V{d}\sqrt{\frac{\log m +\log\frac{2}{\delta}}{2}}\frac{1}{\sqrt{n}}.$$

%These bounds also show that \acp{MRC}' risks  
% the light-tail constant of $P^*$ and $\Phi$ $\V{c}$, and a bound for the solutions of \eqref{learning-eq}. The light-tail constant $\V{c}$ can be taken to be 
%$$c_i=\max_{x\in\set{X},y\in\set{Y}} (\Phi(x,y))_i-\min_{x\in\set{X},y\in\set{Y}}(\Phi(x,y))_i\mbox{,  for $i=1,2,\ldots,m$}$$
%which is equal to $\V{1}$ for functionals taking values between $0$ and $1$, as those described in Section\ref{} and used in Section\ref{}. [other bounds for those functionals?]
%The vector $\V{s}$ above can result in over-pessimistic interval estimates $\V{a}_n$ and $\V{b}_n$ for the expectation of $\Phi$ since it is based on Hoeffding's inequality and the union bound \cite{BouLugMas:13} for the $m$ components of $\Phi$. In practice, \acp{LPC} can be developed by using tighter interval estimates for the expectation of $\Phi$. Such tighter intervals can be obtained for instance by using bootstrapping methods, the central limit theorem, and better estimates of sub-Gaussian parameters than $\V{c}$.

%obtained by better estimates of the sub-Gaussian parameters than $\V{c}$ or by using the central limit theorem to determine asymptotically valid confidence intervals.
The generalization bounds for the risk provided in Theorem~3 of \cite{FarTse:16} and Theorems~2 and 3 of \cite{LeeRag:18} are analogous to those in inequalities \eqref{bound1} and \eqref{bound4} above. In particular, they also show risk's bounds w.r.t. to the minimax risk corresponding to an infinite number of samples. The bounds in \cite{FarTse:16} and \cite{LeeRag:18} correspond to uncertainty sets defined by expectation fits with empirical marginals and Wasserstein distances, respectively, while the bounds \eqref{bound1} and \eqref{bound4} above correspond to the proposed uncertainty sets in \eqref{uncertainty}. The generalization bounds in Corollary~3.2 in \cite{NamDuc:17} and Theorem~2 of \cite{AbaMohKuh:15} are analogous to those in inequalities \eqref{bound1} and \eqref{bound2} above. In particular, they also show how the risk can be upper bounded (assymptotically in \cite{NamDuc:17} and inequality \eqref{bound2} or with certain confidence in \cite{AbaMohKuh:15} and inequality \eqref{bound1}) by the corresponding finite-sample minimax expected loss. The bounds in \cite{NamDuc:17} and  \cite{AbaMohKuh:15} correspond with uncertainty sets defined by f-divergences, and Wasserstein distances, respectively, while the bounds \eqref{bound1} and \eqref{bound2} above correspond with the proposed uncertainty sets defined by linear constraints. 

%\vspace{-0.2cm}
\section{Implementation of \acp{MRC}}\label{sec-4}
%\vspace{-0.2cm}
Algorithm~\ref{codes1} describes \acp{MRC} learning stage that obtains parameters $\B{\mu}^*,\nu^*$ by solving optimization problem \eqref{learning-ineq} in Theorem~\ref{th1} given expectation estimates in \eqref{interval} obtained from training samples. An upper bound for the expected loss is directly obtained as by-product of such optimization while a lower bound for the expected loss requires to solve an additional linear optimization problem given by \eqref{lower} in Theorem~\ref{prop}.%\vspace{-0.2cm}

\begin{algorithm}
\setstretch{1.1}
\caption{\label{codes1}-- Pseudocode for \ac{MRC} learning}
 \begin{tabular}{ll}\hspace{-0.1cm}\textbf{Input:}&\hspace{-0.3cm}Training samples $(x_1,y_1),(x_2,y_2),\ldots,(x_n,y_n)$, width of confidence intervals $\B{\lambda}$\\
%&\hspace{-0.3cm}\\%\vspace{0.1cm}
%&\hspace{-0.3cm}\\
&\hspace{-0.3cm}feature mapping $\Phi$, and matrices $\B{\mathsf{\Phi}}_1,\B{\mathsf{\Phi}}_2,\ldots,\B{\mathsf{\Phi}}_r$ satisfying \eqref{M-matrices}\\
\hspace{-0.1cm}\textbf{Output:}&\hspace{-0.3cm}Parameters $\B{\mu}^*,\nu^*$, upper bound $R^{\V{a}_n,\V{b}_n}$, and [Optional] lower bound $L^{\V{a}_n,\V{b}_n}$\end{tabular}%\\
%&\hspace{-0.3cm} \\
%&\hspace{-0.3cm}$

\begin{algorithmic}[1] 
\setstretch{1.5}
%\vspace{-0.1cm}
		\State $\B{\tau}_n\gets \frac{1}{n}\sum_{i=1}^n\Phi(x_i,y_i)$, $\V{a}_n\gets \B{\tau}_n-\B{\lambda}\frac{1}{\sqrt{n}}$, $\V{b}_n\gets \B{\tau}_n+\B{\lambda}\frac{1}{\sqrt{n}}$
		%\State  
		%\State  %\hspace{-0.16cm}\vspace{-1.0cm}
		\State\hspace{-0.1cm}\vspace{-.83cm}\begin{align*}\hspace{-0.23cm}\begin{array}{lll}
		\B{\mu}_{b}^*,\B{\mu}_{a}^*,\nu^*\gets&\hspace{-.25cm}\underset{\B{\mu}_{a},\B{\mu}_{b},\nu}{\arg\min}&  \V{b}_n^{\text{T}}\B{\mu}_{b}-\V{a}_n^{\text{T}}\B{\mu}_{a}-\nu\\
		&\hspace{-0.2cm}\mbox{s. t.}&\hspace{-0.7cm}(\V{1}_\set{C})^{\text{T}}\left(\B{\mathsf{\Phi}}_i(\B{\mu}_{a}-\B{\mu}_{b})+\nu\V{1}\right) \leq1- |\set{C}|,\  \forall i\in\{1,2,\ldots,r\}, \set{C}\subseteq\set{Y},\set{C}\neq\emptyset\\[0.0cm]
&&\hspace{-0.6cm}\B{\mu}_{a},\B{\mu}_{b}\succeq \V{0} \end{array}\end{align*}
\State  $\B{\mu}^*\gets\B{\mu}_{a}^*-\B{\mu}_{b}^*$, $R^{\V{a}_n,\V{b}_n}\gets\V{b}_n^{\text{T}}\B{\mu}_{b}^*-\V{a}_n^{\text{T}}\B{\mu}_{a}^*-\nu^*$%]\hspace{-0.0cm}\vspace{-0.15cm}[Optional] 
%\State  
\State \mbox{[Optional]}\vspace{-0.85cm}\begin{align*}\hspace{-.4cm}\begin{array}{llll}L^{\V{a}_n,\V{b}_n}\gets&\hspace{-.3cm}-\underset{\B{\mu}_{a},\B{\mu}_{b},\nu}{\min}\hspace{-.3cm}&\V{b}_n^{\text{T}}\B{\mu}_{b}-\V{a}_n^{\text{T}}\B{\mu}_{a}-\nu\\&
\hspace{.25cm}\mbox{s. t.}&\hspace{-.cm}\B{\mathsf{\Phi}}_i(\B{\mu}_{a}-\B{\mu}_{b})+\nu\V{1}\preceq\V{1}-\B{\varepsilon}_{i},\   \forall i\in\{1,2,\ldots,r\}\\[0.0cm]&
&\hspace{-.cm} \B{\mu}_{a},\B{\mu}_{b}\succeq \V{0} \end{array}\end{align*}\vspace{-0.7cm}\hspace{-.5cm}
\begin{align*}&\mbox{where}&\hspace{-0.2cm}\B{\varepsilon}_{i}&=\left\{\begin{array}{cc}
(\B{\mathsf{\Phi}}_i\B{\mu}^*+(\nu^*+1)\V{1}])_+/c_i&\mbox{if }c_i\neq 0\\
\V{1}/|\set{Y}|&\mbox{if }c_i=0\end{array}\right.&
\mbox{\hspace{-0.3cm}and }&c_i&=\|(\B{\mathsf{\Phi}}_i\B{\mu}^*+(\nu^*+1)\V{1})_+\|_1
\end{align*}
%\vspace{-.52cm}
\end{algorithmic}
\end{algorithm}
%\vspace{-0.1cm}

Optimization problems \eqref{learning-ineq} and \eqref{lower} addressed at learning can be efficiently solved; in the following we show equivalent representations of such optimization problems that are appropriate for implementation. For each $x\in\set{X}$, let $\B{\Phi}_x$ be the $|\set{Y}|\times m$ matrix with $y$-th row equal to $\Phi(x,y)^\text{T}$. If $\B{\mathsf{\Phi}}_1,\B{\mathsf{\Phi}}_2,\ldots,\B{\mathsf{\Phi}}_r$ are $r$ matrices describing the range of matrices $\B{\Phi}_x$ for varying $x\in\set{X}$, i.e.,
\begin{align}\label{M-matrices}\{\B{\mathsf{\Phi}}_{i}:\  i=1,2,\ldots,r\}=\{\B{\Phi}_x:\ x\in\set{X}\}\end{align}
then,  constraints in optimization problem \eqref{lower} are equivalent to $2m+r|\set{Y}|$ linear constraints. Constraints in optimization problem \eqref{learning-ineq} are equivalent to $2m$ linear and $r$ nonlinear constraints since
$\|(\B{\Phi}(\B{\mu}_{a}-\B{\mu}_{b})+(\nu+1)\V{1})_+\|_{1,\infty}\leq 1$
is equivalent to 
\begin{align}\label{nonlinear}\|\left(\B{\mathsf{\Phi}}_{i}(\B{\mu}_{a}-\B{\mu}_{b})+(\nu+1)\V{1}\right)_+\|_1\leq 1\  \mbox{for }i=1,2,\ldots,r. \end{align}
Furthermore, constraints in optimization problem \eqref{learning-ineq} are also equivalent to $2m+r(2^{|\set{Y}|}-1)$ linear constraints because \eqref{nonlinear} is equivalent to
\begin{align*}
 (\V{1}_\set{C})^{\text{T}}\left(\B{\mathsf{\Phi}}_i(\B{\mu}_{a}-\B{\mu}_{b})+\nu\V{1}\right) \leq 1-|\set{C}|,\  \forall\   i\in\{1,2,\ldots,r\},\  
   \set{C}\subseteq\set{Y},\set{C}\neq\emptyset
   \end{align*}
since $\|(\B{\mathsf{\Phi}}_i(\B{\mu}_{a}-\B{\mu}_{b})+(\nu+1)\V{1})_+\|_1=\underset{\set{C}\subseteq \set{Y}}{\max}(\V{1}_\set{C})^{\text{T}}\left(\B{\mathsf{\Phi}}_i(\B{\mu}_{a}-\B{\mu}_{b})+(\nu+1)\V{1}\right).$ 
%equals \begin{align*}\max_{\set{S}\subseteq \set{Y}}\,(\V{1}_\set{S})^{\text{T}}\left(\B{\mathsf{\Phi}}_i(\B{\mu}_{a}-\B{\mu}_{b})+(\nu+1)\V{1}\right).\end{align*}

Classification problems with a moderate number of classes $|\set{Y}|$ can benefit by the formulation of \eqref{learning-ineq} as a linear optimization problem with $2m+r(2^{|\set{Y}|}-1)$ constraints instead of that as nonlinear convex optimization with $2m+r$ constraints. The number $r$ of matrices $\B{\mathsf{\Phi}}_1,\B{\mathsf{\Phi}}_2,\ldots,\B{\mathsf{\Phi}}_r$ needed to cover the range of matrices $\B{\Phi}_x$, $x\in\set{X}$, determines the number of constraints in the optimization problems solved for \ac{MRC} learning. Efficient optimization can be achieved using constraint generation techniques or approximations with a subset of constraints. 

\begin{figure}
\psfrag{Y}[l][t][0.65]{\hspace{-3mm}Risk}
\psfrag{X}[l][b][0.65]{\hspace{-10mm}Training size n}
\psfrag{A12345678912345678912345678}[l][][0.52]{\hspace{-15.5mm}\ac{MRC} Risk $R(\up{h}^{\V{a}_n,\V{b}_n})$}
\psfrag{B}[l][][0.52]{\hspace{-0.4mm}Upper bound $R^{\V{a}_n,\V{b}_n}$}
\psfrag{C}[l][][0.52]{\hspace{-0.4mm}Lower bound $L^{\V{a}_n,\V{b}_n}$}
\psfrag{D}[l][][0.52]{\hspace{-0.4mm}Bayes Risk $R_{\text{Bayes}}$}
%\psfrag{50}[l][][0.52]{\hspace{-2mm}$50$}
\psfrag{0.2}[l][][0.4]{\hspace{-3mm}$0.2$}
\psfrag{0.1}[l][][0.4]{\hspace{-3mm}$0.1$}
\psfrag{0}[l][][0.4]{\hspace{-3mm}$0$}
\psfrag{0.3}[l][][0.4]{\hspace{-3mm}$0.3$}
%\psfrag{0.4}[l][][0.5]{\hspace{-3mm}$0.4$}
%\psfrag{0.5}[l][][0.5]{\hspace{-3mm}$0.5$}
%\psfrag{0.6}[l][][0.5]{\hspace{-3mm}$0.6$}
\psfrag{100}[l][][0.4]{\hspace{-2mm}$100$}
\psfrag{1000}[l][][0.4]{\hspace{-4mm}$1000$}
\psfrag{10000}[l][][0.4]{\hspace{-4mm}$10000$}
\psfrag{500}[l][][0.4]{\hspace{-4mm}$500$}
\psfrag{5000}[l][][0.4]{\hspace{-4mm}$5000$}
\centering
\subfigure[\footnotesize{Adult dataset.} ]{\includegraphics[width=0.49\textwidth]{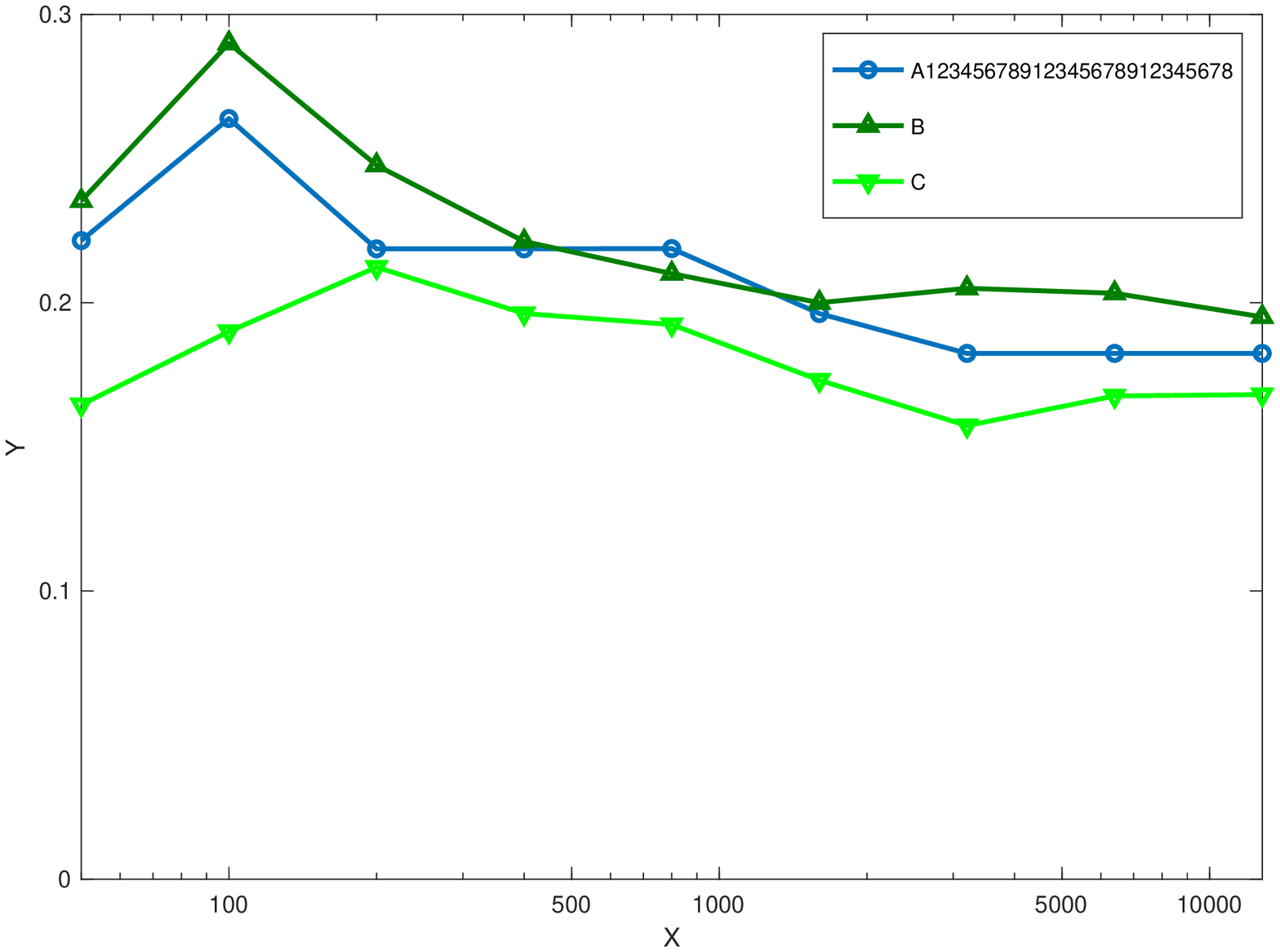}\label{fig_bounds1}}
\subfigure[\footnotesize{Magic dataset.}]{\includegraphics[width=0.49\textwidth]{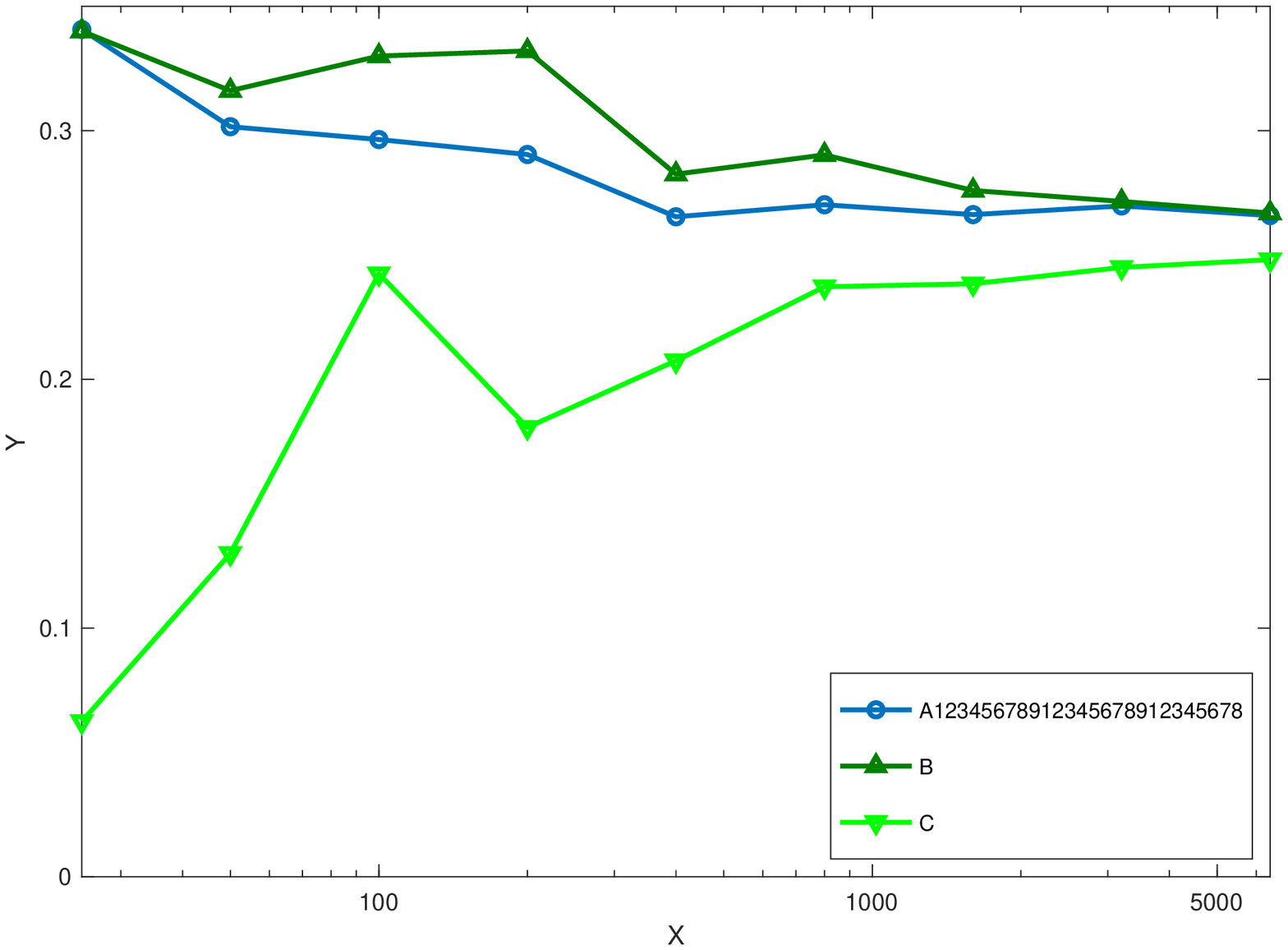}\label{fig_bounds2}}
%\vspace{-0.1cm}
\caption{\small Upper and lower \ac{MRC} risk bounds obtained at learning.}
%\vspace{-0.25cm}
\end{figure}

At prediction stage, \acp{MRC} use the parameters $\B{\mu}^*$ and $\nu^*$ obtained at learning to assign label $y\in\set{Y}$ to instance $x\in\set{X}$ with probability %\vspace{-0.02cm}
\begin{align}
\up{h}^{\V{a},\V{b}}(y|x)=\left\{\begin{array}{cc}
(\Phi(x,y)^{\text{T}}\B{\mu}^*+\nu^*+1)_+/c_x&\mbox{if }c_x\neq 0\\[0.05cm]
1/\set{Y}&\mbox{if }c_x=0\end{array}\right.\label{t-a,b}
\end{align}% \vspace{-0.02cm}
that satisfies \eqref{robust-act} in Theorem~\ref{th1} by taking $c_x=\underset{y\in\set{Y}}{\mathlarger{\sum}}(\Phi(x,y)^{\text{T}}\B{\mu}^*+\nu^*+1)_+$.
%\vspace{-.2cm}
\section{Experimental results}\label{sec-5}
%\vspace{-.2cm}
In this section we show numerical results for \acp{MRC} using 8 UCI datasets for multi-class classification. The first set of results shows the suitability of the upper and lower bounds $R^{\V{a},\V{b}}$ and $L^{\V{a},\V{b}}$ for \acp{MRC} with varying training sizes, while the second set of results compares the classification error of \acp{MRC} w.r.t. state-of-the-art techniques. 

\acp{MRC}' results are obtained using feature mappings given by instances' thresholding, similarly to those used by maximum entropy and logistic regression methods \cite{MehRos:18,DudPhi:04,PhiAnd:06}. Such feature mappings are adequate for a streamlined implementation of \acp{MRC} because they take a reduced number of values.\footnote{The implementation of \acp{MRC} with more sophisticated feature mappings, such as those embedding data into a \ac{RKHS}, can be enabled by using constraint generation techniques or subgradient descent methods.} Let each instance $x\in\set{X}$ be given by $\V{x}=[x^{(1)},x^{(2)},\ldots,x^{(D)}]^{\text{T}}\in\mathbb{R}^D$, and let $\text{Th}_i\in\mathbb{R}$ be a threshold corresponding with dimension $d_i\in\{1,2,\ldots,D\}$ for $i=1,2,\ldots,k$. We consider feature mappings with $m=|\set{Y}|(k+1)$ components corresponding to the different combinations of labels and thresholds. Specifically, 
\begin{align}\label{features}\Phi^{(l)}(x,y)&=\mathbb{I}\left\{y=i\right\}
\mbox{ for }l=(i-1)(k+1)+1,\  i=1,2,\ldots,|\set{Y}|\nonumber\\
\Phi^{(l)}(x,y)&=\mathbb{I}\big\{x^{(d_j)}\leq\text{Th}_{j}\big\}\mathbb{I}\left\{y=i\right\}\nonumber\\
&\mbox{ for }l=(i-1)(k+1)+j+1,\  i=1,2,\ldots,|\set{Y}|, j=1,2,\ldots,k.
\end{align}
We obtain up to $k=200/|\set{Y}|$ thresholds using one-dimensional decision trees (decision stumps) so that the feature mapping has up to $m=200+|\set{Y}|$ components, and we solve the optimization problems at learning with the constraints corresponding to the $r=n$ matrices $\B{\mathsf{\Phi}}_{i}=\B{\Phi}_{x_i}, i=1,2,\ldots,n,$
obtained from the $n$ training instances. For all datasets, interval estimates for feature mapping expectations were obtained using \eqref{interval} with $\lambda^{(i)}=0.25$ for $i=1,2,\ldots,m$.  All other classification techniques were implemented using their default parameters, and the convex optimization problems have been solved using CVX package \cite{GraBoyYe:06}.  %The 

In the first set of experimental results, we use ``Adult'' and ``Magic'' data sets from the UCI repository. For each training size, one instantiation of training samples is used for learning as described in Algorithm~\ref{codes1}, and \ac{MRC}'s risk is estimated using the remaining samples. It can be observed from the Figures \ref{fig_bounds1} and \ref{fig_bounds2} that the lower and upper bounds obtained at learning can offer accurate estimates for the risk without using test samples.

%Figure~\ref{fig_bounds} shows the risk of an \ac{MRC} that uses $\V{a}_n$ and $\V{b}_n$ given by \eqref{interval2} with $s^{(i)}=0.25$, $i=1,2,\ldots,m$. For each training size, one instantiation of training samples is used for training, and \ac{MRC}'s risk is estimated using $1000$ test samples. It can be observed from the figure that the lower and upper bounds obtained at training can offer accurate estimates for the risk without using test samples.

In the second set of experimental results, we use $6$ data sets from the UCI repository (first column of Table~\ref{table:results}). \acp{MRC} are compared with $7$ classifiers: \ac{DT}, \ac{QDA}, \ac{KNN}, Gaussian kernel \ac{SVM}, and \ac{RF}, as well as the related \ac{RRM} classifiers \ac{AMC}, and \ac{MEM}. The first 5 classifiers were implemented using scikit-learn package, \ac{AMC} \cite{FatAnq:16} was implemented with Gaussian kernel using the publicly available code provided by the authors in \url{https://github.com/rizalzaf/adversarial-multiclass}, and \ac{MEM} was implemented as shown in \cite{FarTse:16}. The errors and standard deviations in Table~\ref{table:results} have been estimated using paired and stratified $10$-fold cross validation. The upper and lower bounds showed in columns UB and LB, respectively, are obtained without averaging, that is, by one-time learning \acp{MRC} with all samples. It can be observed from the table that the accuracy of proposed \acp{MRC} is competitive with state-of-the-art techniques even using a simple feature mapping given by instances' thresholding. Table~\ref{table:results} also shows the tightness of the presented performance bounds for assorted datasets. Python code with the proposed \ac{MRC} is provided in \url{https://github.com/MachineLearningBCAM/Minimax-risk-classifiers-NeurIPS-2020} with the settings used in these experimental results.

%\footnote{In the supplementary materials, we provide Python code to reproduce the experimental results, including the implementation of the proposed \ac{MRC}.}

\begin{table*}
\caption{\small Classification error and performance bounds of \ac{MRC} in comparison with state-of-the-art techniques.}
%\vspace{0.1cm}
\label{table:results}
\centering
\def\arraystretch{1.3}
\resizebox{\textwidth}{!}{%
\begin{tabular}{l|ccc|cccccccc}%\\[-0.45cm]
\toprule Data set&LB& MRC&UB&QDA &DT& KNN& SVM& RF&AMC&MEM\vspace{-0.05cm}\\[0.cm]
\hline
Mammog.&$.16$&$.18\,$\raisebox{1.5pt}{$\mathsmaller{\pm}$}$\,.04$&$.21$&$.20\,$\raisebox{1.5pt}{$\mathsmaller{\pm}$}$\,.04$& $.24\,$\raisebox{1.5pt}{$\mathsmaller{\pm}$}$\,.04$& $.22\,$\raisebox{1.5pt}{$\mathsmaller{\pm}$}$\,.04$& $.18\,$\raisebox{1.5pt}{$\mathsmaller{\pm}$}$\,.03$& $.21\,$\raisebox{1.5pt}{$\mathsmaller{\pm}$}$\,.06$& $.18\,$\raisebox{1.5pt}{$\mathsmaller{\pm}$}$\,.03$ &$.22\,$\raisebox{1.5pt}{$\mathsmaller{\pm}$}$\,.04$\\
Haberman&$.24$&$.27\,$\raisebox{1.5pt}{$\mathsmaller{\pm}$}$\,.03$&$.27$&$.24\,$\raisebox{1.5pt}{$\mathsmaller{\pm}$}$\,.03$& $.39\,$\raisebox{1.5pt}{$\mathsmaller{\pm}$}$\,.14$& $.30\,$\raisebox{1.5pt}{$\mathsmaller{\pm}$}$\,.07$& $.26\,$\raisebox{1.5pt}{$\mathsmaller{\pm}$}$\,.04$& $.35\,$\raisebox{1.5pt}{$\mathsmaller{\pm}$}$\,.12$&  $.25\,$\raisebox{1.5pt}{$\mathsmaller{\pm}$}$\,.04$ &$.27\,$\raisebox{1.5pt}{$\mathsmaller{\pm}$}$\,.02$\\
Indian liv.& $.28$&$.29\,$\raisebox{1.5pt}{$\mathsmaller{\pm}$}$\,.01$&$.30$ &$.44\,$\raisebox{1.5pt}{$\mathsmaller{\pm}$}$\,.08$& $.35\,$\raisebox{1.5pt}{$\mathsmaller{\pm}$}$\,.09$& $.34\,$\raisebox{1.5pt}{$\mathsmaller{\pm}$}$\,.05$& $.29\,$\raisebox{1.5pt}{$\mathsmaller{\pm}$}$\,.02$& $.30\,$\raisebox{1.5pt}{$\mathsmaller{\pm}$}$\,.05$& $.29\,$\raisebox{1.5pt}{$\mathsmaller{\pm}$}$\,.01$ &$.29\,$\raisebox{1.5pt}{$\mathsmaller{\pm}$}$\,.01$\\
Diabetes& $.22$&$.26\,$\raisebox{1.5pt}{$\mathsmaller{\pm}$}$\,.03$&$.28$&$.26\,$\raisebox{1.5pt}{$\mathsmaller{\pm}$}$\,.03$ & $.29\,$\raisebox{1.5pt}{$\mathsmaller{\pm}$}$\,.07$& $.26\,$\raisebox{1.5pt}{$\mathsmaller{\pm}$}$\,.05$& $.24\,$\raisebox{1.5pt}{$\mathsmaller{\pm}$}$\,.04$& $.26\,$\raisebox{1.5pt}{$\mathsmaller{\pm}$}$\,.05$&  $.24\,$\raisebox{1.5pt}{$\mathsmaller{\pm}$}$\,.04$ &$.34\,$\raisebox{1.5pt}{$\mathsmaller{\pm}$}$\,.04$\\
Credit&  $.12$&$.15\,$\raisebox{1.5pt}{$\mathsmaller{\pm}$}$\,.18$&$.17$&$.22\,$\raisebox{1.5pt}{$\mathsmaller{\pm}$}$\,.07$& $.22\,$\raisebox{1.5pt}{$\mathsmaller{\pm}$}$\,.14$& $.14\,$\raisebox{1.5pt}{$\mathsmaller{\pm}$}$\,.09$& $.16\,$\raisebox{1.5pt}{$\mathsmaller{\pm}$}$\,.17$& $.17\,$\raisebox{1.5pt}{$\mathsmaller{\pm}$}$\,.15$& $.15\,$\raisebox{1.5pt}{$\mathsmaller{\pm}$}$\,.18$&$.14\,$\raisebox{1.5pt}{$\mathsmaller{\pm}$}$\,.04$\\
Glass& $.22$&$.36\,$\raisebox{1.5pt}{$\mathsmaller{\pm}$}$\,.08$&$.47$&$.64\,$\raisebox{1.5pt}{$\mathsmaller{\pm}$}$\,.04$& $.39\,$\raisebox{1.5pt}{$\mathsmaller{\pm}$}$\,.18$& $.34\,$\raisebox{1.5pt}{$\mathsmaller{\pm}$}$\,.08$& $.34\,$\raisebox{1.5pt}{$\mathsmaller{\pm}$}$\,.11$&$.40\,$\raisebox{1.5pt}{$\mathsmaller{\pm}$}$\,.14$& $.42\,$\raisebox{1.5pt}{$\mathsmaller{\pm}$}$\,.14$ & $.35\,$\raisebox{1.5pt}{$\mathsmaller{\pm}$}$\,.08$\\
\hline\vspace{-5mm}\\
Avg. rank & &2.7&& 5.1& 7.0& 3.8& 2.0& 5.3&  2.5&3.8\\[-0.05cm]
\bottomrule
\end{tabular}}%\vspace{-0.3cm}
\end{table*}
\normalsize

%\vspace{-.2cm}
\section{Conclusion}
%\vspace{-.2cm}
The proposed \acp{MRC} minimize the worst-case expected 0-1 loss over general classification rules, and provide performance guarantees at learning. The paper also describes \acp{MRC}' implementation in practice, and presents their finite-sample generalization bounds. Experimentation with benchmark datasets shows the reliability and tightness of the presented performance bounds, and the competitive classification performance of \acp{MRC} with simple feature mappings given by thresholds. %The results presented show that minimax supervised classification with 0-1 loss and performance guarantees can be enabled by linear optimization problems given by expectation estimates obtained from training data.  
The results presented show that supervised classification does not require to choose a surrogate loss that substitutes original 0-1 loss, and a specific family that constraints classification rules. Differently from conventional techniques,
the inductive bias exploited by \acp{MRC} comes only from a feature mapping that  serves to constrain the distributions considered. Learning with \acp{MRC} is achieved without further design choices by solving linear optimization problems that can also provide tight performance guarantees.

\clearpage

\section*{Broader Impact}
%\vspace{-0.2cm}
The results presented in the paper can enable new approaches for supervised learning that can benefit general applications of supervised classification. Such results do not put anybody at a disadvantage, create consequences in case of failure or leverage biases in the data.
%\small
\begin{ack}
Funding in direct support of this work has been provided by the Spanish Ministry of Economy and Competitiveness MINECO through Ramon y Cajal Grant RYC-2016-19383, BCAM's Severo Ochoa Excellence Accreditation SEV-2017-0718, Project PID2019-105058GA-I00, and Project TIN2017-82626-R, and by the Basque Government through the ELKARTEK and BERC 2018-2021 programmes.
\end{ack}

%\vspace{-0.3cm}
\bibliography{IEEEabrv,StringDefinitions,BiblioCV,WGroup,bib-santi}
\bibliographystyle{unsrt}
%\vspace{-0.2cm}
%\bibliographystyle{aaai}
%\bibliographystyle{plainnat}
%\bibliographystyle{IEEEtran}
%\bibliographystyle{plain}
\newpage
\normalsize
\onecolumn

%\appendices
\appendix
\section*{Appendices}

\section{Auxiliary lemmas}
The proofs of Theorem~\ref{th1} and Theorem~\ref{prop} require the lemmas provided below.

\begin{lemma}\label{lemma-dual}
The norms $\|\cdot\|_{\infty,1}$ and $\|\cdot\|_{1,\infty}$ are dual. 
\end{lemma}
\begin{proof}

The dual norm of $\|\cdot\|_{\infty,1}$ assigns each $\V{w}\in\mathbb{R}^{|\set{I}||\set{J}|}$ for finite sets $\set{I}$ and $\set{J}$, the real number
$$\sup_{\V{v}:\  \|\V{v}\|_{\infty,1}\leq 1}\V{w}^{\text{T}}\V{v}.$$
We have that for $\V{v}$ with $\|\V{v}\|_{\infty,1}\leq 1$
\begin{align*}\V{w}^{\text{T}}\V{v}&=\sum_{i\in\set{I}}\sum_{j\in\set{J}}w_{(i,j)}v_{(i,j)}\leq\sum_{i\in\set{I}}\sum_{j\in\set{J}}|w_{(i,j)}||v_{(i,j)}|\\
&\leq \sum_{i\in\set{I}}\left(\max_j|v_{(i,j)}|\right)\sum_{j\in\set{J}}|w_{(i,j)}|\leq\max_{i\in\set{I}}\sum_{j\in\set{J}}|w_{(i,j)}|\sum_{i\in\set{I}}\left(\max_j|v_{(i,j)}|\right)\\&=\|\V{w}\|_{1,\infty}\|\V{v}\|_{\infty,1}\leq\|\V{w}\|_{1,\infty}\end{align*}
So, to prove the result we just need to find a vector $\V{u}$ such that $\|\V{u}\|_{\infty,1}\leq 1$ and $\V{w}^{\text{T}}\V{u}=\|\V{w}\|_{1,\infty}$. Let $\iota\in\arg\max_{i\in\set{I}}\sum_{j\in\set{J}}|w_{(i,j)}|$, then $\V{u}$ given by 
$$u_{(i,j)}=\left\{\begin{array}{cc}1&\mbox{ if }i=\iota \mbox{ and } w_{(i,j)}\geq 0\\
-1&\mbox{ if }i=\iota \mbox{ and } w_{(i,j)}< 0\\
0&\mbox{ otherwise }\end{array}\right.$$
satisfies $\|\V{u}\|_{\infty,1}\leq 1$ and $\V{w}^{\text{T}}\V{u}=\|\V{w}\|_{1,\infty}$.

\end{proof}

\begin{lemma}\label{lemma-conjugate}
Let $\V{u}\in\mathbb{R}^{|\set{I}||\set{J}|}$ for finite sets $\set{I}$ and $\set{J}$, and $f_1$, $f_2$ be the functions $f_1(\V{v})=\|\V{v}\|_{\infty,1}-\V{1}^{\text{T}}\V{v}+I_+(\V{v})$ and $f_2(\V{v})=\V{v}^{\text{T}}\V{u}+I_+(\V{v})$ for $\V{v}\in\mathbb{R}^{|\set{I}||\set{J}|}$, where
$$I_+(\V{v})=\left\{\begin{array}{cc}0 &\mbox{if}\  \V{v}\succeq\V{0} \\\infty&\mbox{otherwise}\end{array}\right..$$
Then, their conjugate functions are
$$f_1^*(\V{w})=\left\{\begin{array}{cc}0 &\mbox{if}\  \|(\V{1}+\V{w})_+\|_{1,\infty}\leq 1 \\\infty&\mbox{otherwise}\end{array}\right.$$
$$f_2^*(\V{w})=\left\{\begin{array}{cc}0 &\mbox{if}\  \V{w}\preceq\V{u}  \\\infty&\mbox{otherwise}\end{array}\right..$$

%where $\|\cdot\|_*$ denotes the dual norm
%$$\|\V{y}\|_*=\sup\{\V{y}^{\text{T}}\V{x}:\  \|\V{x}\|\leq 1\}.$$
\end{lemma}

\begin{proof}
	By definition of conjugate function we have
$$f_1^*(\V{w}) =\sup_{\V{v}} (\V{w}^{\text{T}}\V{v} - \|\V{v}\|_{\infty,1} +\V{1}^{\text{T}}\V{v}- I_+(\V{v})) = \sup_{\V{v} \succeq 0} ((\V{1}+\V{w})^{\text{T}}\V{v} -  \|\V{v}\|_{\infty,1}).$$
%	Now we can distinguish two cases. 
\begin{itemize}
\item If $\|(\V{1}+\V{w})_+\|_{1,\infty} \leq 1$, for each $\V{v} \succeq \V{0}$, $\V{v} \neq \V{0}$ we have
	$$(\V{1}+\V{w})^{\text{T}}\V{v} \leq ((\V{1}+\V{w})_+)^{\text{T}}\V{v} = \|\V{v}\|_{\infty,1} \left(((\V{1}+\V{w})_+)^{\text{T}}\frac{\V{v}}{\|\V{v}\|_{\infty,1}}\right)$$ 
	and by definition of dual norm we get
	$$(\V{1}+\V{w})^{\text{T}}\V{v} \leq \|\V{v}\|_{\infty,1} \|(\V{1}+\V{w})_+\|_{1,\infty} \leq \|\V{v}\|_{\infty,1}$$
	which implies
	$$ (\V{1}+\V{w})^{\text{T}}\V{v} - \|\V{v}\|_{\infty,1}\leq 0.$$
	Moreover, $(\V{1}+\V{w})^{\text{T}}\V{0} - \|\V{0}\|_{\infty,1}= 0$, so we have that $f_1^*(\V{w}) = 0$. 
\item If $\|(\V{1}+\V{w})_+\|_{1,\infty} > 1$, by definition of dual norm and using Lemma~\ref{lemma-dual} there exists $\V{u}$ such that $((\V{1}+\V{w})_+)^{\text{T}}\V{u} > 1$ and $\|\V{u}\|_{\infty,1} \leq 1$. Define $\tilde{\V{u}}$ as
	$$
	\tilde{u}_{(i,j)} =\left\{\begin{array}{cc}
	u_{(i,j)} & \text{ if } u_{(i,j)} \geq 0 \text{ and } 1+w_{(i,j)} \geq 0 \\
	0 & \text{ if } u_{(i,j)} < 0 \text{ or } 1+w_{(i,j)} < 0
	\end{array}\right.
	$$
	By definition of $\tilde{\V{u}}$ and $\|\cdot\|_{\infty,1}$ we have
	$$\|\tilde{\V{u}}\|_{\infty,1} \leq \|\V{u}\|_{\infty,1} \leq 1 $$
	and
	$$(\V{1}+\V{w})^{\text{T}}\tilde{\V{u}} = ((\V{1}+\V{w})_+)^{\text{T}}\tilde{\V{u}}\geq((\V{1}+\V{w})_+)^{\text{T}}\V{u} > 1. $$
	Now let $t > 0$ and take $\V{v} = t \tilde{\V{u}} \succeq 0$, then we have
	$$(\V{1}+\V{w})^{\text{T}}\V{v} - \|\V{v}\|_{\infty,1} = t \left ( (\V{1}+\V{w})^{\text{T}}\tilde{\V{u}} - \|\tilde{\V{u}}\|_{\infty,1} \right ) $$
	which tends to infinity as $t \to + \infty$ because $(\V{1}+\V{w})^{\text{T}}\tilde{\V{u}} - \|\tilde{\V{u}}\|_{\infty,1} > 0$, so we have that $f_1^*(\V{w}) = + \infty$.
	\end{itemize}
	Finally, the expression for $f_2^*$ is straightforward since	$$f_2^*(\V{w})=\sup_{\V{v}\succeq\V{0}}((\V{w}-\V{u})^{\text{T}}\V{v}).$$
\end{proof}

\section{Proof of Theorem~\ref{th1}}\label{proof-th1}
Let set $\widetilde{\set{U}}$ and function $\widetilde{\ell}(\up{h},\up{p})$ be given by
$$\widetilde{\set{U}}=\{\up{p}:\set{X}\times\set{Y}\to\mathbb{R}\mbox{ s.t. }  \V{p}\succeq \V{0},\  \|\V{p}\|_{1,\infty}\leq 1\}$$
%$$\widetilde{\ell}(T,P)=\ell(T,P)+([\V{1},\B{\Phi}^{\text{T}}]P-[1,\B{\tau}^{\text{T}}])\B{\lambda}^*$$
$$\widetilde{\ell}(\up{h},\up{p})=\V{b}^{\text{T}}\B{\mu}_{b}^*-\V{a}^{\text{T}}\B{\mu}_{a}^*-\nu^*+\V{p}^{\text{T}}(\B{\Phi}(\B{\mu}_{a}^*-\B{\mu}_{b}^*)+(\nu^*+1)\V{1}-\V{h}).$$
In the first step of the proof we show that $\up{h}^{\V{a},\V{b}}$ satisfying \eqref{robust-act} is a solution of optimization problem $\min_{\up{h}\in T(\set{X},\set{Y})}\max_{\up{p}\in\widetilde{\set{U}}}\widetilde{\ell}(\up{h},\up{p})$, and in the second step of the proof we show that a solution of $\min_{\up{h}\in T(\set{X},\set{Y})}\max_{\up{p}\in\widetilde{\set{U}}}\widetilde{\ell}(\up{h},\up{p})$ is also a solution of $\min_{\up{h}\in T(\set{X},\set{Y})}\max_{\up{p}\in\set{U}^{\V{a},\V{b}}}\ell(\up{h},\up{p})$.

%For the first step, note that $\widetilde{\ell}(T,P)=1+\gamma^*+\V{a}^{\text{T}}\B{\alpha}^*-\V{b}^{\text{T}}\B{\beta}^*+\sum_{x\in\set{X}}\widetilde{\ell}^x(T_x,P_x)$ with
For the first step, note that
$$\widetilde{\ell}(\up{h},\up{p})=\V{b}^{\text{T}}\B{\mu}_{b}^*-\V{a}^{\text{T}}\B{\mu}_{a}^*-\nu^*+\sum_{x\in\set{X}}\V{p}_x^{\text{T}}\left(\B{\Phi}_x(\B{\mu}_{a}^*-\B{\mu}_{b}^*)+(\nu^*+1)\V{1}-\V{h}_x\right).$$

%$$\breve{\ell}_x(T_x,P_x))=-\V{p}_x^{\text{T}}\left(\V{t}_x+\V{1}\gamma^*+\B{\Phi}_x(\B{\alpha}^*-\B{\beta}^*)\right)$$

Then, optimization problem $\min_{\up{h}\in T(\set{X},\set{Y})}\max_{\up{p}\in\widetilde{\set{U}}}\widetilde{\ell}(\up{h},\up{p})$ is equivalent to
$$\begin{array}{ccc}\min &\max &  \sum_{x\in\set{X}}\V{p}_x^{\text{T}}\left(\B{\Phi}_x(\B{\mu}_{a}^*-\B{\mu}_{b}^*)+(\nu^*+1)\V{1}-\V{h}_x\right)\\
\up{h}_x\in\Delta(\set{Y})\  \forall x\in\set{X}&\V{p}_x\succeq \V{0}, \|\V{p}_x\|_1\leq 1 \forall x\in\set{X}&\end{array}$$

%_{T_x\in\Delta(\set{Y}),\  \forall x\in\set{X}}\  \max_{P:\  \V{p}_x\succeq \V{0}, \|\V{p}_x\|_1\leq 1 \forall x\in\set{X}}\  \sum_{x\in\set{X}}\widetilde{\ell}_x(T(x,\cdot),P(x,\cdot))$$

%where $$\set{V}=\{P\in\Delta(\set{X}\times\set{Y}):\  \V{p}_x\succeq \V{0}\mbox{ and } \|\V{p}_x\|_1\leq 1\}$$

%$$\begin{array}{cc}\min&\\
%\mbox{s. t.}&T(x)\succeq \V{0}\\
%&\sum_{x,y}T(x,y)=1,\  \forall x\in\set{X}\end{array}\begin{array}{cc}\sum_{x\in\set{X}}\max&\widetilde{\ell}_x(T(x),P(x,\cdot))\\
%\mbox{s. t.}&P(x,\cdot)\in\widetilde{U}_x\end{array}$$
that is separable and has solution given by
$$\begin{array}{cccc}\up{h}_x^{\V{a},\V{b}}\in&\arg\min &\max &\V{p}_x^{\text{T}}\left(\B{\Phi}_x(\B{\mu}_{a}^*-\B{\mu}_{b}^*)+(\nu^*+1)\V{1}-\V{h}_x\right)\\&
\up{h}_x\in\Delta(\set{Y})&\V{p}_x\succeq \V{0}, \|\V{p}_x\|_1\leq 1&\end{array}$$
%$$T_x^*\in\arg\min_{T_x\in\Delta(\set{Y})}\max_{R\in\set{V}_x}\widetilde{\ell}_x(Q,R)$$
%$$T_{\B{\lambda}^*}(x)\in\arg\begin{array}{cc}\min&\\
%\mbox{s. t.}&Q\in\Delta(\set{Y})\end{array}\begin{array}{cc}\max&\widetilde{\ell}_x(Q,R)\\
%\mbox{s. t.}&R\in\widetilde{U}_x\end{array}$$
for each $x\in\set{X}$. The inner maximization above is given in closed-form by
$$\begin{array}{cc}&\underset{\V{p}_x\succeq \V{0}, \|\V{p}_x\|_1\leq 1}{\max} \V{p}_x^{\text{T}}\left(\B{\Phi}_x(\B{\mu}_{a}^*-\B{\mu}_{b}^*)+(\nu^*+1)\V{1}-\V{h}_x\right)\vspace{0.2cm}\\
&\hspace{0.85cm}=\|\left(\B{\Phi}_x(\B{\mu}_{a}^*-\B{\mu}_{b}^*)+(\nu^*+1)\V{1}-\V{h}_x\right)_+\|_\infty\geq 0\end{array}$$
that takes its minimum value $0$ for any $\V{h}_x^{\V{a},\V{b}}\succeq \B{\Phi}_x(\B{\mu}_{a}^*-\B{\mu}_{b}^*)+(\nu^*+1)\V{1}$.

%note that the solution of the above optimization for each $x\in\set{X}$ is achieved by any $\V{t}_x^*\succeq \V{1}\gamma^*+\B{\Phi}_x(\B{\alpha}^*-\B{\beta}^*)$ because 

%\\
%\mbox{s. t.}&R\in\widetilde{U}_x\end{array}=\max_{y\in\set{Y}}([1,\B{\Phi}(x,y)]\B{\lambda}^*-Q(y))^+$$
%and $\max_{y\in\set{Y}}(a(y)-b(y))^+\geq 0$ and equals $0$ if $b\succeq a$. 

For the second step, if $\up{h}^{\V{a},\V{b}}$ is a solution of $\min_{\up{h}\in T(\set{X},\set{Y})}\max_{\up{p}\in\widetilde{\set{U}}}\widetilde{\ell}(\up{h},\up{p})$ we have that
%$$\max_{P\in\set{U}}\min_{T\in\Delta(X,Y)}\ell(T,P)\leq\max_{P\in\set{U}}\ell(T^*},P)\leq\max_{P\in\widetilde{\set{U}}}\widetilde{\ell}(T^*},P)$$
\begin{align}\label{ineqs}\min_{\up{h}\in T(\set{X},\set{Y})}\max_{\up{p}\in\widetilde{\set{U}}}\widetilde{\ell}(\up{h},\up{p})=\max_{\up{p}\in\widetilde{\set{U}}}\widetilde{\ell}(\up{h}^{\V{a},\V{b}},\up{p})\geq\max_{\up{p}\in\set{U}^{\V{a},\V{b}}}\ell(\up{h}^{\V{a},\V{b}},\up{p})\geq\min_{\up{h}\in T(\set{X},\set{Y})}\max_{\up{p}\in\set{U}^{\V{a},\V{b}}}\ell(\up{h},\up{p})\end{align}
where the first inequality is due to the fact that $\set{U}^{\V{a},\V{b}}\subset\widetilde{\set{U}}$ and $\widetilde{\ell}(\up{h},\up{p})\geq\ell(\up{h},\up{p})$ for 
$\up{p}\in\set{U}^{\V{a},\V{b}}$ because 
$$\V{b}^{\text{T}}\B{\mu}_{b}^*-\V{a}^{\text{T}}\B{\mu}_{a}^*+\V{p}^{\text{T}}\B{\Phi}(\B{\mu}_{a}^*-\B{\mu}_{b}^*)\leq 0$$ by definition of $\set{U}^{\V{a},\V{b}}$ and since $\B{\mu}_{a}^*,\B{\mu}_{b}^*\succeq \V{0}$.

%Since $\ell(h,p)$ is bounded, $\set{X}\times\set{Y}$ is finite, and both $\set{U}_\Phi^{\V{a},\V{b}}$ and $\Delta(X,Y)$ 
%are closed and convex,

Since $\ell(\up{h},\up{p})$ is continuous and convex-concave, and both $\set{U}^{\V{a},\V{b}}$ and $ T(\set{X},\set{Y})$ 
are convex and compact, the min and the max in $R^{\V{a},\V{b}}=\min_{\up{h}\in T(\set{X},\set{Y})}\max_{\up{p}\in\set{U}^{\V{a},\V{b}}}\ell(\up{h},\up{p})$ can be interchanged (see e.g., \cite{GruDaw:04}) and we have that  $R^{\V{a},\V{b}}=\max_{\up{p}\in\set{U}^{\V{a},\V{b}}}\min_{\up{h}\in T(\set{X},\set{Y})}\ell(\up{h},\up{p})$. In addition,
$$\min_{\up{h}\in T(\set{X},\set{Y})}\ell(\up{h},\up{p})=\min_{\up{h}\in T(\set{X},\set{Y})}\V{p}^{\text{T}}(\V{1}-\V{h})=\V{p}^{\text{T}}\V{1}-\|\V{p}\|_{\infty,1}$$
because the optimization problem above is separable for $x\in\set{X}$ and
\begin{align}\label{entropy}\max_{\up{h}_x\in\Delta(\set{Y})}\V{p}_x^{\text{T}}\V{h}_x=\|\V{p}_x\|_\infty.\end{align}

Then $R^{\V{a},\V{b}}=\max_{\up{p}\in\set{U}^{\V{a},\V{b}}}\V{p}^{\text{T}}\V{1}-\|\V{p}\|_{\infty,1}$ that can be written as
\begin{align}\begin{array}{cc}\underset{\V{p}}{\max}&\V{p}^{\text{T}}\V{1}-\|\V{p}\|_{\infty,1}-I_+(\V{p})\\
\mbox{s. t.} &-\V{p}^{\text{T}}\V{1}=-1\\&
\V{a}\preceq\B{\Phi}^{\text{T}}\V{p}\preceq\V{b}\end{array}\label{opt}\end{align}
where 
$$I_+(\V{p})=\left\{\begin{array}{cc}0 &\mbox{if}\  \V{p}\succeq \V{0}\\\infty&\mbox{otherwise}\end{array}\right.$$

The Lagrange dual of the optimization problem \eqref{opt} is
\begin{align}\label{dual}\begin{array}{cc}\min&\V{b}^{\text{T}}\B{\mu}_{b}-\V{a}^{\text{T}}\B{\mu}_{a}-\nu+f^*\left(\B{\Phi}(\B{\mu}_{a}-\B{\mu}_{b})+\nu\V{1}\right)\\\B{\mu}_{a},\B{\mu}_{b}\in\mathbb{R}^{m},\nu\in\mathbb{R}&\\\mbox{s.t.}&
\B{\mu}_{a}\succeq \V{0}, \B{\mu}_{b}\succeq \V{0}\end{array}\end{align}
where $f^*$ is the conjugate function of $f(\V{p})=\|\V{p}\|_{\infty,1}-\V{p}^{\text{T}}\V{1}+I_+(\V{p})$ (see e.g., section 5.1.6 in \cite{BoyVan:04}). Then, optimization problem \eqref{dual} becomes \eqref{learning-ineq} using the Lemma~\ref{lemma-conjugate} above.
%$$\begin{array}{cc}\max&[1,\B{\tau}^{\text{T}}]\B{\lambda}\\
%\mbox{s. t.}&\|([\V{1},\B{\Phi}^{\text{T}}]\B{\lambda})^+\|_{1,\infty}\leq 1\end{array}$$

Strong duality holds between optimization problems \eqref{opt} and \eqref{learning-ineq} since constraints in \eqref{opt} are affine. Then, if $\B{\mu}_{a}^*,\B{\mu}_{b}^*,\nu^*$ is a solution of \eqref{learning-ineq} we have that $R^{\V{a},\V{b}}$ is equal to the value of
\begin{align}\label{opt2}\max_{\up{p}}\V{p}^{\text{T}}\V{1}-\|\V{p}\|_{\infty,1}-I_+(\V{p})-(\V{p}^{\text{T}}\B{\Phi}-\V{b}^{\text{T}})\B{\mu}_{b}^*+(\V{p}^{\text{T}}\B{\Phi}-\V{a}^{\text{T}})\B{\mu}_{a}^*+(\V{p}^{\text{T}}\V{1}-1)\nu^*\end{align} that equals
$$\max_{\up{p}\in\widetilde{\set{U}}}\V{p}^{\text{T}}\V{1}-\|\V{p}\|_{\infty,1}+\V{b}^{\text{T}}\B{\mu}_{b}^*-\V{a}^{\text{T}}\B{\mu}_{a}^*-\nu^*+\V{p}^{\text{T}}\left(\B{\Phi}(\B{\mu}_{a}^*-\B{\mu}_{b}^*)+\nu^*\V{1}\right)$$
%\begin{align}\label{opt2}\min_P\|P\|_{\infty,1}+I^+(P)-\sum_{i=1}^{m+1}\lambda_i^*(\text{Tr}(\B{\Phi}_i^{\text{T}}\V{P})-\tau_i)\end{align} 
%that equals
%$$=\min_{P\in\widetilde{U}}\|P\|_{\infty,1}-\sum_{i=1}^{m+1}\lambda_i^*(\text{Tr}(\B{\Phi}_i^{\text{T}}\V{P})-\tau_i)$$
since a solution of the primal problem \eqref{opt} belongs to $\widetilde{\set{U}}$ and is also a solution of \eqref{opt2}. Therefore, 
%$$=\begin{array}{cc}\min&\|P\|_{\infty,1}-([\V{1},\B{\Phi}^{\text{T}}]P-[1,\B{\tau}^{\text{T}}])\B{\lambda}^*\\
%\mbox{s. t.}&P\in\widetilde{\set{U}} \end{array}$$
\begin{align*}R^{\V{a},\V{b}}&=\max_{\up{p}\in\widetilde{\set{U}}}\min_{\up{h}\in T(\set{X},\set{Y})}\ell(\up{h},\up{p})+\V{b}^{\text{T}}\B{\mu}_{\B{b}}^*-\V{a}^{\text{T}}\B{\mu}_{a}^*-\nu^*+\V{p}^{\text{T}}\left(\B{\Phi}(\B{\mu}_{a}^*-\B{\mu}_{b}^*)+\nu^*\V{1}\right)\\&=\max_{\up{p}\in\widetilde{\set{U}}}\min_{\up{h}\in T(\set{X},\set{Y})}\widetilde{\ell}(\up{h},\up{p})=\min_{\up{h}\in T(\set{X},\set{Y})}\max_{\up{p}\in\widetilde{\set{U}}}\widetilde{\ell}(\up{h},\up{p})\end{align*}
%$$R_\Phi^{\V{a},\V{b}}=\max_{P\in\widetilde{\set{U}}}\min_{T\in T(\set{X},\set{Y}}\ell(T,P)+\sum_{i=1}^{m+1}\lambda_i^*(\text{Tr}(\B{\Phi}_i^{\text{T}}\V{P})-\tau_i)=\max_{P\in\widetilde{\set{U}}}\min_{T\in\Delta(X,Y)}\widetilde{\ell}(T,P)=\min_{T\in\Delta(X,Y)}\max_{P\in\widetilde{\set{U}}}\widetilde{\ell}(T,P)$$
where the last equality is due to the fact that $\widetilde{\ell}(\up{h},\up{p})$ is continuous and convex-concave, and both $\widetilde{\set{U}}$ and $T(\set{X},\set{Y})$ are convex and compact. Then, inequalities in \eqref{ineqs} are in fact equalities and $\up{h}^{\V{a},\V{b}}$ is solution of $\min_{\up{h}\in T(\set{X},\set{Y})}\max_{\up{p}\in\set{U}^{\V{a},\V{b}}}\ell(\up{h},\up{p})$.

%is bounded, $\set{X}\times\set{Y}$ is finite, and both $\widetilde{\set{U}}$ and $\Delta(X,Y)$ are closed and convex. Then, inequalities in \eqref{ineqs} are in fact equalities and $h^*$ is solution of $\min_{h\in\Delta(X,Y)}\max_{p\in\set{U}_\Phi^{\V{a},\V{b}}}\ell(h,p)$.

\section{Proof of Theorem~\ref{prop}}\label{proof-prop}
The result is a direct consequence of the fact that for any $\up{p}\in\set{U}^{\V{a},\V{b}}$
$$\min_{\widetilde{\up{p}}\in\set{U}^{\V{a},\V{b}}}\ell(\up{h},\widetilde{\up{p}})\leq\ell(\up{h},\up{p})\leq\max_{\widetilde{\up{p}}\in\set{U}^{\V{a},\V{b}}}\ell(\up{h},\widetilde{\up{p}})$$
and 
$$\min_{\widetilde{\up{p}}\in\set{U}^{\V{a},\V{b}}}\ell(\up{h},\widetilde{\up{p}})=\min_{\widetilde{\up{p}}\in\set{U}^{\V{a},\V{b}}}\V{\widetilde{p}}^{\text{T}}(\V{1}-\V{h})$$
$$\max_{\widetilde{\up{p}}\in\set{U}^{\V{a},\V{b}}}\ell(\up{h},\widetilde{\up{p}})=-\min_{\widetilde{\up{p}}\in\set{U}^{\V{a},\V{b}}}\V{\widetilde{p}}^{\text{T}}(\V{h}-\V{1}).$$

The expression for $\kappa^{\V{a},\V{b}}(q)$ in \eqref{lower} is obtained since
\begin{align}\label{opt_kappa}\begin{array}{ccc}\underset{\widetilde{\up{p}}\in\set{U}^{\V{a},\V{b}}}{\min}\V{\widetilde{p}}^{\text{T}}(-\V{q})=&\underset{\widetilde{\V{p}}}\min&\V{\widetilde{p}}^{\text{T}}(-\V{q})+I_+(\V{\widetilde{p}})\\&
\mbox{s. t.} &-\V{1}^{\text{T}}\V{\widetilde{p}}=-1\\&&
\V{a}\preceq\B{\Phi}^{\text{T}}\V{\widetilde{p}}\preceq\V{b}\end{array}\end{align}
where 
$$I_+(\V{\widetilde{p}})=\left\{\begin{array}{cc}0 &\mbox{if}\  \V{\widetilde{p}}\succeq \V{0}\\\infty&\mbox{otherwise}\end{array}\right.$$
Then, the Lagrange dual of the optimization problem \eqref{opt_kappa} is
\begin{align}\begin{array}{cc}\max&\V{a}^{\text{T}}\B{\mu}_{a}-\V{b}^{\text{T}}\B{\mu}_{b}+\nu-f^*\left(\B{\Phi}(\B{\mu}_{a}-\B{\mu}_{b})+\nu\V{1}\right)\\\B{\mu}_{a},\B{\mu}_{b}\in\mathbb{R}^{m},\nu\in\mathbb{R}&\\\mbox{s.t.}&
\B{\mu}_{a}\succeq \V{0}, \B{\mu}_{b}\succeq \V{0}\end{array}\end{align}
where $f^*$ is the conjugate function of $f(\V{\widetilde{p}})=\V{\widetilde{p}}^{\text{T}}(-\V{q})+I^+(\V{\widetilde{p}})$ 
that leads to \eqref{lower} using Lemma~\ref{lemma-conjugate}.

\section{Proof of Theorem~\ref{th-bounds}}\label{proof-bounds}
Firstly, with probability at least $1-\delta$ we have that $\up{p}^*\in\set{U}^{\V{a}_n,\V{b}_n}$ and 
$$\|\B{\tau}_{\infty}-\B{\tau}_{n}\|_2\leq \|\V{d}\|_2\sqrt{\frac{\log m+\log\frac{2}{\delta}}{2n}}$$
because, using Hoeffding's inequality \cite{BouLugMas:13} we have that for $i=1,2,\ldots,m$
$$\mathbb{P}\left\{|\tau_{\infty,i}-\tau_{n,i}|< t_i\right\}\geq1-2\exp\left\{-\frac{2n^2t_i^2}{nd_i^2}\right\}$$
so taking $t_i=d_i\sqrt{\frac{\log m+\log\frac{2}{\delta}}{2n}}$ we get
$$\mathbb{P}\left\{|\tau_{\infty,i}-\tau_{n,i}|< d_i\sqrt{\frac{\log m+\log\frac{2}{\delta}}{2n}}\right\}\geq1-2\exp\left\{-\log m-\log\frac{2}{\delta}\right\}= 1-\frac{\delta}{m}$$
%$$=1-2\exp\left\{-\log m-\log\frac{2}{\delta}\right\}= 1-\frac{\delta}{m}$$
%for $t=c_i\sqrt{\frac{\log m+\log\frac{2}{\delta}}{2n}}$
and using the union bound we have that 
\begin{align*}\mathbb{P}\Bigg\{|\tau_{\infty,i}-\tau_{n,i}|< d_i\sqrt{\frac{\log m+\log\frac{2}{\delta}}{2n}},\  &i=1,2,\ldots,m\Bigg\}
\\
&\geq 1-m+\sum_{i=1}^m\mathbb{P}\left\{|\tau_{\infty,i}-\tau_{n,i}|< d_i\sqrt{\frac{\log m+\log\frac{2}{\delta}}{2n}}\right\}\\&\geq 1-\delta.
\end{align*}
%Hence, with probability at least $1-\delta$.%, $P^*\in\set{U}_{\Phi}^{\V{a}_n,\V{b}_n}$ because
%$$\mathbb{P}\left\{|\tau_i^*-\tau_{n,i}|\leq c_i\sqrt{\frac{\log m+\log\frac{2}{\delta}}{2n}},i=1,2,\ldots,m\right\}\geq 1-m+\sum_{i=1}^m\mathbb{P}\left\{|\tau_i^*-\tau_{n,i}|\leq c_i\sqrt{\frac{\log m+\log\frac{2}{\delta}}{2n}}\right\}\geq 1-\delta$$
%and 
%$$\|\B{\tau}^*-\B{\tau}_{n}\|_2\leq \|\V{c}\|_2\sqrt{\frac{\log m+\log\frac{2}{\delta}}{2n}}$$ because 

%with probability at least $1-\delta$

For the first inequality in \eqref{bound1}, we have that $R(\up{h}^{\V{a}_n,\V{b}_n})\leq R^{\V{a}_n,\V{b}_n}$ with probability at least $1-\delta$ since $\up{p}^*\in\set{U}^{\V{a}_n,\V{b}_n}$ with probability at least $1-\delta$. 

For the second inequality in \eqref{bound1}, let $\B{\mu}^*,\nu^*$ be the solution with minimum euclidean norm of \eqref{learning-eq} for $\V{a}=\B{\tau}_\infty$; $\left[(\B{\mu}^*)^+,(-\B{\mu}^*)^+,\nu^*\right]$ is a feasible point of \eqref{learning-ineq} because $\B{\mu}^*=(\B{\mu}^*)^+-(-\B{\mu}^*)^+$ and $\B{\mu}^*,\nu^*$ is a feasible point of \eqref{learning-eq}. Hence
$$R^{\V{a}_n,\V{b}_n}\leq \V{b}_n^{\text{T}}(-\B{\mu}^*)^+-\V{a}_n^{\text{T}}(\B{\mu}^*)^+-\nu^*=R^{\B{\tau}_{\infty}}+(\V{b}_n-\B{\tau}_{\infty})^{\text{T}}(-\B{\mu}^*)^++(\B{\tau}_{\infty}-\V{a}_n)^{\text{T}}(\B{\mu}^*)^+$$
$$=R^{\B{\tau}_{\infty}}-\left(\B{\tau}_{\infty}-\B{\tau}_n-\V{d}\sqrt{\frac{\log m+\log\frac{2}{\delta}}{2n}}\right)^{\text{T}}(-\B{\mu}^*)^++\left(\B{\tau}_{\infty}-\B{\tau}_n+\V{d}\sqrt{\frac{\log m+\log\frac{2}{\delta}}{2n}}\right)^{\text{T}}(\B{\mu}^*)^+$$
$$=R^{\B{\tau}_{\infty}}+(\B{\tau}_n-\B{\tau}_{\infty})^{\text{T}}\B{\mu}^*+\sqrt{\frac{\log m+\log\frac{2}{\delta}}{2n}}\V{d}^{\text{T}}((\B{\mu}^*)^++(-\B{\mu}^*)^+)$$
Then the result is obtained using Cauchy-Schwarz inequality and the fact that $\|(\B{\mu}^*)^++(-\B{\mu}^*)^+\|_2=\|\B{\mu}^*\|_2$. 
%We have that $P^*\in\set{U}_{\Phi}^{\V{a}_n,\V{b}_n}$ if $-\frac{\V{c}(\delta)}{\sqrt{2n}}\preceq\B{\tau}^*-\B{\tau}_n\preceq\frac{\V{c}(\delta)}{\sqrt{2n}}$, that is, if $|\tau_i^*-\tau_{n,i}|\leq \frac{c_i(\delta)}{\sqrt{2n}}$ for $i=1,2,\ldots,m$. Using Hoeffding's inequality [cite] we have that
%$$\mathbb{P}\{|\tau_i^*-\tau_{n,i}|\leq \frac{c_i(\delta)}{\sqrt{2n}} \}\geq 1-2\exp\{-\log m-\log\frac{2}{\delta}\}=1-\frac{\delta}{m}$$
%and hence $P^*\in\set{U}_{\Phi}^{\V{a}_n,\V{b}_n}$ with probability larger than $1-\delta$ using the union bound [cite]

%$$R_\Phi^{\V{a},\V{b}}\leq 1-\V{a}_n^{\text{T}}(\B{\alpha}^*)^++\V{b}_n^{\text{T}}(-\B{\alpha}^*)^+-\gamma^*=R_\Phi^{\B{\tau}^*}+(\B{\tau}^*-\V{a}_n)^{\text{T}}(\B{\alpha}^*)^++(\V{b}_n-\B{\tau}^*)^{\text{T}}(-\B{\alpha}^*)^+$$

%$$\leq R_\Phi^{\B{\tau}^*}+\left\|\B{\tau}^*-\B{\tau}_n\|_2+\frac{\|\V{c}(\delta)\right\|_2\}{\sqrt{2n}}\|\B{\alpha}^*\|_2\leq R_\Phi^{\B{\tau}^*}+2\left(\|\B{\tau}^*-\B{\tau}_n\|_\infty+\frac{\|\V{c}(\delta)\|_\infty}{\sqrt{2n}}\right)\|\B{\alpha}^*\|_\infty$$
%Then the first result is obtained using again Hoeffding's inequality and the union bound.
For the result in \eqref{bound2}, note that using Theorem~\ref{prop} and since $\up{p}^*\in\set{U}^{\V{a}_n,\V{b}_n}$ with probability at least $1-\delta$ we have that 
$$R(\up{h}^{\B{\tau}_n})\leq\max_{\up{p}\in\set{U}^{\V{a}_n,\V{b}_n}}\ell(\up{h}^{\B{\tau}_n},\up{p})=\underset{\B{\Phi}(\B{\mu}_{a}-\B{\mu}_{a})+\nu\V{1}\preceq \V{h}^{\B{\tau}_n}-\V{1}}{\min}\V{b}_n^{\text{T}}\B{\mu}_{b}-\V{a}_n^{\text{T}}\B{\mu}_{a}-\nu$$
so that, if $\B{\mu}_n^*,\nu_n^*$ is the solution with minimum euclidean norm of \eqref{learning-eq} for $\V{a}=\B{\tau}_n$, we have that $R(\up{h}^{\B{\tau}_n})\leq \V{b}_n^{\text{T}}(-\B{\mu}_n^*)^+-\V{a}_n^{\text{T}}(\B{\mu}_n^*)^+-\nu_n^*$ because $\B{\mu}_n^*=(\B{\mu}_n^*)^+-(-\B{\mu}_n^*)^+$ and $\B{\Phi}\B{\mu}_n^*+\nu_n^*\V{1}\preceq\V{h}^{\tau_n}-\V{1}$ by definition of $\V{h}^{\tau_n}$. Therefore, the result is obtained since 
\begin{align*}R(\up{h}^{\B{\tau}_n})\leq\,& \left(\B{\tau}_n+\V{d}\sqrt{\frac{\log m+\log\frac{2}{\delta}}{2n}}\right)^{\text{T}}(-\B{\mu}_n^*)^+-\left(\B{\tau}_n-\V{d}\sqrt{\frac{\log m+\log\frac{2}{\delta}}{2n}}\right)^{\text{T}}(\B{\mu}_n^*)^+-\nu_n^*\\
=\,& R^{\B{\tau}_n}+\V{d}^{\text{T}}\sqrt{\frac{\log m+\log\frac{2}{\delta}}{2n}}\left((\B{\mu}_n^*)^++(-\B{\mu}_n^*)^+\right).
\end{align*}

%For the third result, let 
%$$\tilde{\B{\lambda}}_n,\tilde{\gamma}_n\in\arg\underset{\B{\Phi}^{\text{T}}\B{\lambda}+\gamma\preceq-\V{h}^{\tau_n}}{\max} \B{\tau}_n^{\text{T}}\B{\lambda}+\gamma$$
%then $L_\Phi^{\B{\tau}_n}=1+\B{\tau}_n^{\text{T}}\tilde{\B{\lambda}}_n+\tilde{\gamma}_n$ and the result is obtained since $p^*\in\set{U}_\Phi^{\V{a}_n,\V{b}_n}$ with probability at least $1-\delta$ and hence
%\begin{align}R(h^{\B{\tau}_n})&\geq 1+\underset{\B{\Phi}^{\text{T}}(\B{\alpha}-\B{\beta})+\gamma\preceq-\V{h}^{\tau_n},\B{\alpha},\B{\beta}\succeq \V{0}}{\max} \V{a}_n^{\text{T}}\B{\alpha}-\V{b}_n^{\text{T}}\B{\beta}+\gamma\geq 1+\V{a}_n^{\text{T}}(\tilde{\B{\lambda}}_n)^+-\V{b}_n^{\text{T}}(-\tilde{\B{\lambda}}_n)^++\tilde{\gamma}_n\nonumber\\
%&=1+\B{\tau}_n^{\text{T}}\left((\tilde{\B{\lambda}}_n)^+-(-\tilde{\B{\lambda}}_n)^+\right)+\tilde{\gamma}_n-\V{c}^{\text{T}}\sqrt{\frac{\log m+\log\frac{2}{\delta}}{2n}}\left((\tilde{\B{\lambda}}_n)^++(-\tilde{\B{\lambda}}_n)^+\right)
%\end{align}

For the result in \eqref{bound4}, note that using Theorem~\ref{prop} and since $\up{p}^*\in\set{U}^{\B{\tau}_{\infty}}$ we have that
$$R(\up{h}^{\B{\tau}_n})\leq\max_{\up{p}\in\set{U}^{\B{\tau}_{\infty}}}\ell(\up{h}^{\B{\tau}_n},\up{p})=\underset{\B{\Phi}\B{\mu}+\nu\V{1}\preceq \V{h}^{\B{\tau}_n}-\V{1}}{\min}-(\B{\tau}_{\infty})^{\text{T}}\B{\mu}-\nu$$
so that, if $\B{\mu}_n^*,\nu_n^*$ is the solution with minimum euclidean norm of \eqref{learning-eq} for $\V{a}=\B{\tau}_n$, we have that $R(\up{h}^{\B{\tau}_n})\leq -(\B{\tau}_{\infty})^{\text{T}}\B{\mu}_n^*-\nu_n^*$ because $\B{\Phi}\B{\mu}_n^*+\nu_n^*\V{1}\preceq \V{h}^{\tau_n}-\V{1}$ by definition of $\V{h}^{\tau_n}$.  Let  $\B{\mu}^*,\nu^*$ be the solution with minimum euclidean norm of \eqref{learning-eq} for $\V{a}=\B{\tau}_{\infty}$, the result is obtained since
\begin{align}R(\up{h}^{\B{\tau}_n})&\leq -(\B{\tau}_{\infty})^{\text{T}}\B{\mu}_n^*
-\nu_n^*+\B{\tau}_n^{\text{T}}\B{\mu}_n^*-\B{\tau}_n^{\text{T}}\B{\mu}_n^*+(\B{\tau}_{\infty})^{\text{T}}\B{\mu}^*+\nu^*-(\B{\tau}_{\infty})^{\text{T}}\B{\mu}^*-\nu^*\nonumber\\
&=(\B{\tau}_n-\B{\tau}_{\infty})^{\text{T}}\B{\mu}_n^*+R^{\B{\tau}_{\infty}}-\B{\tau}_n^{\text{T}}\B{\mu}_n^*-\nu_n^*+(\B{\tau}_{\infty})^{\text{T}}\B{\mu}^*+\nu^*\nonumber\\
&\leq (\B{\tau}_n-\B{\tau}_{\infty})^{\text{T}}\B{\mu}_n^*+(\B{\tau}_{\infty}-\B{\tau}_n)^{\text{T}}\B{\mu}^*+R^{\B{\tau}_{\infty}}\label{ineq1}\\
&\leq\|\B{\tau}_n-\B{\tau}_{\infty}\|_2\|\B{\mu}_n^*-\B{\mu}^*\|_2+R^{\B{\tau}_{\infty}}\nonumber
\end{align}
where \eqref{ineq1} is due to the fact that $-\B{\tau}_n^{\text{T}}\B{\mu}_n^*-\nu_n^*\leq-\B{\tau}_n^{\text{T}}\B{\mu}^*-\nu^*$ since $\B{\mu}^*,\nu^*$ is a feasible point of \eqref{learning-eq} for $\V{a}=\B{\tau}_n$.

\end{document}